\documentclass[12pt,a4paper]{article}
\usepackage{geometry}
\usepackage[utf8]{inputenc}
\usepackage[T1]{fontenc}
\usepackage{lmodern} 
\usepackage[english]{babel}
\usepackage{amsmath}
\usepackage{amsfonts}
\usepackage{amssymb}
\usepackage{amsthm}
\usepackage{enumitem}

\usepackage{balance} 
\usepackage{tikz}
\usetikzlibrary{shapes.geometric, arrows, calc, positioning,arrows.meta,bending,positioning}
\tikzstyle{s-arrow} = [thick, dashed, ->, >=stealth]
\tikzstyle{a-arrow} = [thick, ->, >=stealth]
\tikzstyle{c-arrow} = [thick, dotted, ->, >=stealth]

\newtheorem{theorem}{Theorem}
\newtheorem{lemma}{Lemma}

\newtheorem{proposition}[theorem]{Proposition}

\newtheorem{observation}[theorem]{Observation}
\newtheorem{definition}{Definition}
\newtheorem{example}{Example}

\newcommand{\BibTeX}{\rm B\kern-.05em{\sc i\kern-.025em b}\kern-.08em\TeX}

\makeatletter
\newcommand{\superimpose}[2]{{%
		\ooalign{%
			\hfil$\m@th#1\@firstoftwo#2$\hfil\cr
			\hfil$\m@th#1\@secondoftwo#2$\hfil\cr
		}%
}}
\makeatother

\newcommand{\supp}{\mathbin{\mathpalette\superimpose{{\to}{\color{white}{\boldsymbol{\cdot\cdot}}}}}}
\newcommand{\nsupp}{\mathbin{\mathpalette\superimpose{{\supp}{\hspace{-0.3cm}\not}}}}

\renewcommand{\phi}{\varphi}
\renewcommand{\l}{\langle}
\renewcommand{\r}{\rangle}
\renewcommand{\arg}[2]{\l\{#1\},#2\r}

\newcommand{\att}{\to}

\newcommand{\inc}{\overline}
\newcommand{\emptinc}{\overline{\phantom{A}}}

\newcommand{\sarg}{Arg_s}
\newcommand{\warg}{Arg_w}
\newcommand{\uncuts}[1]{\inc{n(#1)}}

\newcommand{\lang}{\mathcal{L}}
\newcommand{\langg}{\lang= \l L,\ \emptinc\ , n \r}
\newcommand{\af}{\mathcal{A}}
\newcommand{\aff}{\af=\l A,\att\r}
\newcommand{\baf}{\mathcal{B}}
\newcommand{\baff}{\baf=\l A,\att,\supp\r}
\newcommand{\sbaf}{\mathcal{SB}}
\newcommand{\sbaff}{\sbaf=\l\lang,A,\att,\supp\r}

\title{Rejecting Arguments Based on Doubt in Structured Bipolar Argumentation}

\author{Michael A. M\"uller\footnote{Université de Fribourg, Switzerland; CRIL CNRS Univ Artois, Lens, France.}, Srdjan Vesic\footnote{CRIL CNRS Univ Artois, Lens, France.}, and Bruno Yun\footnote{Univ Lyon, UCBL, CNRS, INSA Lyon, LIRIS, UMR5205, F-69622 Villeurbanne, France.}}

\begin{document}

	\maketitle 
	
	\begin{abstract}
		This paper develops a new approach to computational argumentation that is informed by philosophical and linguistic views. Namely, it takes into account two ideas that have received little attention in the literature on computational argumentation: First, an agent may rationally reject an argument based on mere doubt, thus not all arguments they could defend must be accepted; and, second, that it is sometimes more natural to think in terms of which individual sentences or claims an agent accepts in a debate, rather than which arguments. In order to incorporate these two ideas into a computational approach, we first define the notion of structured bipolar argumentation frameworks (SBAFs), where arguments consist of sentences and we have both an attack and a support relation between them. Then, we provide semantics for SBAFs with two features: (1) Unlike with completeness-based semantics, our semantics do not force agents to accept all defended arguments. (2) In addition to argument extensions, which give acceptable sets of arguments, we also provide semantics for language extensions that specify acceptable sets of sentences. These semantics represent reasonable positions an agent might have in a debate. Our semantics lie between the admissible and complete semantics of abstract argumentation. Further, our approach can be used to provide a new perspective on existing approaches. For instance, we can specify the conditions under which an agent can ignore support between arguments (i.e. under which the use of abstract argumentation is warranted) and we show that deductive support semantics is a special case of our approach.
	\end{abstract}

		Abstract Argumentation; Doubt Bipolar; Argumentation; Structured Argumentation; Interpreting Semantics

	
	\section{Introduction}

For most people, the primary contact point with arguments is through discussions and debates, be it in person or online. In these contexts, they are confronted with the following central question: Given all the arguments, what should you believe? One way to answer this question is through computational argumentation \cite{Dung1995acceptability, Baroni2018handbook}. We can build models of debates and use them to evaluate which sets of arguments, called \textit{extensions}, an agent considering the debate should. This can help agents navigate complex and potentially inconsistent information. The literature has produced a wide range of \textit{semantics} that determine the acceptability of arguments (see \cite{Baroni2018abstract}). Many of these approaches come out of an interest in modelling different ways of formal reasoning based on specific logics, logic programming, or how agents can reason with inconsistent or uncertain knowledge \cite{Nouioua2013afs,AmgoudCayrol1998on,Dung1995acceptability,Arieli2021logic,Bondarenko1993assumption,Potyka2018continous,Pu2015attacker}.

This paper uses computational approaches such as abstract \cite{Dung1995acceptability}, structured \cite{Besnard2014introduction}, and bipolar \cite{CayrolLagasquie-Schiex2005on} argumentation as its starting point, but incorporates aspects of informal approaches to argumentation such as philosophy \cite{Toulmin2003uses,Hamblin1970fallacies,Johnson2006logical,Walton2008argumentation}, linguistics \cite{Koszowy2022pragmatic}, and communication studies \cite{VanEemeren1984speech,VanEemeren2004systematic}. Amongst these approaches, we can find two general ideas that are not always explicitly taken into account in formal approaches: (1) In formal approaches, rational agents are assumed to automatically accept all defended arguments, while in informal ones they are allowed to reject such an argument, e.g.\ if it relies on a premise the agents deem very implausible. (2) Formal approaches often require agents to evaluate arguments as a whole, meaning they often take the sentences of an argument (its premises and conclusion) into account only to determine its relations to other arguments. Informal approaches, in contrast, often assume that agents evaluate directly which sentences they should accepted. While some computational approaches implement aspects of these ideas \cite{PolbergOren2014revisting,Gonzalez2021labeled,Toni2014tutorial,Bernreiter2024effect,CocarascuToni2017mining}, we take both of them fully into account.

The idea of rejecting defended arguments comes from the notion of critical reactions \cite{KrabbeVanLaar2011ways,VanLaarKrabbe2012burden}. While an agent has to react critically to an argument in order to reject it, a critical reaction does not have to consist of a counterargument. For instance, instead of attacking an argument, the agent can also challenge it. Such a challenge expresses \textit{doubt} about a premise of the argument and if that doubt cannot be overcome, the argument can be rejected. Thus, even in the absence of any attacks, agents can reject arguments if doubted premises are not supported. Our approach is accordingly one of \textit{bipolar argumentation} \cite{CayrolLagasquie-Schiex2005on}, where arguments can be both attacked and supported. The use of an explicit support relation is also in line with empirical results on how we intuitively deal with arguments \cite{Koszowy2022pragmatic,PolbergHunter2018empirical}. The following example illustrates the novelty and expressiveness of our approach.

{ 
	\begin{example}\label{ex:stradivarius}
		Consider the following situation, loosely inspired by \cite{AmgoudVesic2009repairing}, where an agent considers a debate containing the following arguments: $a_1:$ ``This violin is a Stradivarius since Alex says so'', $a_2:$ ``This violin is expensive since it is a Stradivarius'', $a_3:$ ``We know Clara says that Anne-Sophie owns this violin since she is cited in a newspaper saying so.'', $a_4:$ ``Anne-Sophie owns this violin since Clara says so'', and $a_5:$ ``Hilary owns this violin since Diego says so''. This gives the following situation of bipolar argumentation, where $\att$ are attacks and $\supp$ are supports.
		
		\begin{center}
			\begin{tikzpicture}[node distance=0.5cm]
				
				\node (A1) {$a_1$};
				\node (A2) [right=of A1] {$a_2$};
				\node (A3) [right=of A2] {$a_3$};
				\node (A4) [right=of A3] {$a_4$};
				\node (A5) [right=of A4] {$a_5$};
				
				\draw [s-arrow] (A1) -- (A2);
				\draw [s-arrow] (A3) -- (A4);
				\draw [a-arrow] (A4) -- (A5);
				\draw [a-arrow] (A5) -- (A4);
				
			\end{tikzpicture}
		\end{center}
		
		If the agent uses any form of complete semantics \cite{Dung1995acceptability,Yu2023principle}, they have to accept all defended arguments. Thus any such agent has to accept at least $a_1$ and $a_2$. However, one might reasonably doubt whether Alex really made the claim in $a_1$. The agent might suspect that Alex does not know anything about violins and thus his claim strikes them as odd. This suspicion does not amount to a full counterargument that could be added to the framework, but it is enough for the agent to doubt $a_1$ and reject it if there is no support to remove their doubts. 
		
		The next weaker semantics are based on admissibility \cite{Dung1995acceptability}. Admissibility allows agents to doubt $a_1$, but it also removes any notion of support to $a_2$, as the two arguments can be accepted fully independently of each other. Using d-admissibility of deductive support semantics (see \cite{Boella2010support,CayrolLagasquie-Schiex2013bipolarity} and Section \ref{sect:bip}), support is taken into account and $a_1$ can only be accepted together with $a_2$, giving the following extensions: $\emptyset, \{a_1,a_2\},\{a_3,a_4\},\{a_5\},\{a_1,a_2,a_5\}$, and $\{a_1,a_2,a_3,a_4\}$. But now also $a_3$ can only be accepted together with $a_4$. However, as $a_4$ and $a_5$ contradict each other, one might want to abstain from accepting either of them while still accepting $a_3$. That is, an agent might want to accept that Clara claims that Anne-Sophie owns the violin based on the newspaper while suspending judgment on whether she really owns it. D-admissibility does not allow for this option.
		
		One of the semantics we define, weak coherence (see Section \ref{sect:sem}), captures these dynamics that elude both completeness-based semantics and d-admissibility by allowing agents to accept the extension $\{a_3\}$. A full analysis of this case is provided in Example \ref{ex:full-stradivarius}.
		
	\end{example}
}

The idea of doubting individual sentences also illustrates how informal approaches often evaluate arguments on the level of their sentences. For people confronted with arguments, it is often more intuitive to think about what sentences they accept rather than which arguments \cite{Betz2010theorie}. The latter can be complex structures and it can be challenging to analyse them and figure out to what you are committed when accepting them. Taking into account sentences can help explain why some agent accepts (or should accept) an argument extension, namely by referring to the set of sentences they agree with. This requires taking into account the structure of arguments and accordingly our approach takes the perspective of \textit{structured argumentation}.

{
	\begin{example}
		It might be intuitive for agents to indicate which individual sentences they accept without thinking about the arguments directly. In Example \ref{ex:stradivarius}, we might accept the sentences $Str:$ ``This violin is a Stradivarius'', $Exp$: ``This violin is expensive'', $Cla:$ ``Clara says that Anne-Sophie owns this violin'', and $Hil:$ ``Hilary owns this violin'' (see Example \ref{ex:lang} for a full translation). Is $\{Str,Exp,Cla,Hil\}$ an acceptable set of sentences? In the terms we introduce in Section \ref{sect:sem}, we can understand this as a weakly coherent language extension and it corresponds to the argument extension $\{a_2\}$.
	\end{example}
}

In this paper, we define \textit{structured bipolar argumentation frameworks} (SBAFs) and provide a range of semantics for them. At the argument level, we define two versions of what we call \textit{coherent} semantics. They contrast with completeness-based semantics in that they sometimes allow rejection of defended arguments, but they also contrast with (abstract) admissible semantics in that they take support between arguments into account. We additionally provide two versions of \textit{adequate} semantics that operate on sets of sentences directly. This way, we offer a new perspective on computational argumentation which is informed by informal theories. 

The paper is organised as follows. We first sketch the familiar preliminaries from abstract argumentation (Section \ref{sect:pre}) and then introduce SBAFs in Section \ref{sect:sbaf}. Both argument semantics, which evaluate sets of arguments, and language semantics, which evaluate sets of sentences, are defined in Section \ref{sect:sem}. These semantics are then used to interpret existing semantics from abstract and bipolar argumentation in Section \ref{sect:interpreting}. Finally, we discuss related approaches in the literature (Section \ref{sect:related}).


\section{Preliminaries}\label{sect:pre}

We briefly recall the definitions of abstract argumentation \cite{Dung1995acceptability,Baroni2018abstract}.

\begin{definition}[Abstract Argumentation Framework]\label{def:abstractArgumentation}
	An abstract argumentation framework (AF) is a tuple $\aff$ where $A$ is a finite set of arguments and ${\att}\subseteq A\times A$ an attack relation.
\end{definition}

We write $a\to b$ in case $(a,b)\in{\att}$ and generalise to sets of arguments, i.e.\ $E\to b$ in case $a\to b$ for some $a\in E$ and $a\to E$ in case $a\to b$ for some $b\in E$. Further, we write $a\not\att b$ if $(a,b)\not\in{\att}$. For two sets of arguments, $E\att E'$ means $a\att b$ for some $a\in E$ and $b\in E'$. A set of arguments $E\subseteq A$ \textit{defends} an argument $a\in A$ if $\forall b\in A:b\att a\implies E\att b$.

The semantics are as usual. Let $\aff$ be an AF. An extension $E\subseteq A$ is called
\textit{conflict-free} if $\forall a,b\in E:a\not\att b$,
\textit{admissible} if it is conflict-free and defends all its arguments,
\textit{complete} if it is admissible and it contains all arguments it defends,
\textit{preferred} if it is $\subseteq$-maximal among admissible extensions.

Note the difference between admissible and complete semantics. Whereas admissible extensions are never forced to include an argument, complete semantics requires accepting all defended arguments, including any unattacked ones. Even when it comes to bipolar semantics (see Section \ref{sect:bip}), the focus still lies on variants of complete semantics (cf.\ \cite{Yu2023principle}). The semantics we define in Section \ref{sect:sem} fall in the gap between admissible and complete semantics. 


\section{Structured Bipolar Argumentation}\label{sect:sbaf}


In this section, we first introduce structured bipolar argumentation frameworks and then go on to define their semantics.

\subsection{Frameworks}

Now we start introducing the frameworks for structured bipolar argumentation. We take a structured approach and as such we first have to define the language which we use to represent arguments.

\begin{definition}[Language] \label{def:SBAFlanguage2}
	A language $\lang= \l L,\ \emptinc\ , n \r$ consists of a non-empty set of sentences $L$, a (partial) incompatibility function $\emptinc:L\to 2^{L}$, and a (partial) naming function $n:2^{L}\times L\to{L}$.
	
	We assume $\emptinc$ to be symmetric, i.e. $\forall s,t\in L:s\in\inc{t}\iff t\in\inc{s}$. Additionally, we assume that $\inc{n(\langle\{t\},t\rangle)}=\emptyset.$
\end{definition}

The set of sentences can be any set of objects. All the structure we need to represent arguments is given by the incompatibility and the naming functions. Incompatibility is used to model conflict between sentences in that it associates each sentence with the set of sentences incompatible with it. It is a symmetric notion of contrariness, cf.\ \cite{Toni2014tutorial}. Intuitively, two incompatible sentences should not be accepted together. The naming function allows us to talk about arguments within the language. We represent arguments as a tuple of a set of sentences (the premises) and another sentence (the conclusion), see Definition \ref{def:args}. Thus, $n$ takes arguments and gives them names which express the claim that one can infer the conclusion from the premises, cf.\ \cite{Modgil2014aspic}. The main use of this is to define undercutting attacks (Definition \ref{def:SuppAtt}) using sentences that are incompatible with the name of an argument. Finally, the condition $\inc{n(\langle\{t\},t\rangle)}=\emptyset$ states that an argument that uses the same sentence as its single premise and as its conclusion (see Definition \ref{def:args}) cannot be undercut, as there are no circumstances where we cannot infer a sentence from itself. Note that we do not assume any logical structure or consequence relation on the set of sentences.

{
	\begin{example}\label{ex:lang}
		Let us continue Example \ref{ex:stradivarius} and consider a language with $L=\{Ale,Str,Exp,New,Cla,Ann,Die,Hil\}$, corresponding to: $Ale:$ ``Alex says this violin is a Stradivarius'', $Str:$ ``This violin is a Stradivarius'', $Exp:$ ``This violin is expensive'', $New:$ ``Clara is cited in a newspaper article mentioning that Anne-Sophie owns this violin'', $Cla:$ ``Clara says that Anne-Sophie owns this violin'', $Ann:$ ``Anne-Sophie owns this violin'', $Die:$ ``Diego says Hilary owns this violin'', and $Hil:$ ``Hilary owns this violin'', with $Ann\in\inc{Hil}$ and $Hil\in\inc{Ann}$. We can also add $n(\arg{Ale}{Str}):$ ``One can infer that this violin is a Stradivarius from Alex saying so''.
	\end{example}
}

As in abstract argumentation, we do not construct arguments. Rather, we take them as given, e.g.\ through a debate. This means that we do not need any logic or inference rules that determine from which sentences we can infer others. Formally, any combination of premises and conclusion could be an argument. In the following definition, we use $Prem$ as a function from arguments to sets of sentences in order to indicate the premises of an argument and $Conc$ as a function from arguments to sentences in order to indicate the conclusion of an argument. For an argument $a$, $Prem(a)$ is its set of sentences and $Conc(a)$ is its conclusion.

\begin{definition}[Arguments]\label{def:args}
	An argument in a language $\langg$ is a tuple $a=\l Prem(a), Conc(a)\r$ where 
	$Prem(a)$ is a non-empty finite subset of $L$ and $Conc(a)\in L.$
	
	We also define the set of sentences of an argument $a$ as $Sent(a):=Prem(a)\cup\{Conc(a)\}$. The set of sentences generalises to sets of arguments $Sent(E)=\bigcup_{a\in E}Sent(a)$.
	
	We say $a$ is a minimal argument for sentence $s\in L$ if $a=\langle\{s\},s\rangle$ (see e.g.\ \cite{Modgil2014aspic}).
\end{definition}

{
	\begin{example}\label{ex:args}
		In Example \ref{ex:lang}, we already saw one argument: $a_1:\arg{Ale}{Str}$, which corresponds to ``This violin is a Stradivarius since Alex says so''. The other arguments are: $a_2:\arg{Str}{Exp}$, $a_3:\arg{New}{Cla}$, $a_4:\arg{Cla}{Ann}$, and $a_5:\arg{Die}{Hil}$. An example of a minimal argument would be $\arg{Str}{Str}$, i.e.\ ``This violin is a Stradivarius since this violin is a Stradivarius''. It is an edge case of an argument as it does not contain any real inference step, but they are useful to represent single sentences in frameworks.
	\end{example}
}

Next, we define supports and attacks between arguments. Attacks are defined such that an argument attacks another argument if its conclusion is incompatible with some part of it \cite{Modgil2014aspic}. We call an attack on the name of an argument an undercut and we say that a set of arguments $E$ contains \textit{undercutting information} for an argument $a$ if $\uncuts{a}\cap Sent(E)\neq\emptyset$. Note that attacks are define only through the conclusions of arguments. Accordingly, a set of arguments can contain undercutting information for an argument without attacking it and an argument containing a premise incompatible with the premise of another argument does not necessarily attack it. For instance, there are no attacks between the arguments $a:\arg{r}{x}$ and $b:\arg{y}{z}$ with $r\in\uncuts{b}$ and $y\in\inc{r}$, since the conflicts occur only between their premises. However, $a$ contains undercutting information for $b$. While there are more fine-grained notions of attack in the literature \cite{Corsi2025attack}, they are often variations of the attacks defined here.

The support relation requires more explanation. We want to capture situations where accepting some set of arguments commits you to accepting another: If you accept an argument $a_1:\arg{s}{t}$ and there is an argument $a_2:\arg{t}{u}$, you should also accept $a_2$ since you accept all its premises. When it comes to arguments with more than one premise, only a set of arguments can potentially force its acceptance. For instance, accepting $a_1$ does not force acceptance of $a_3:\arg{t,v}{w}$. But if you also accept $a_4:\arg{v}{v}$, then you should accept $a_3$. Thus, we define support not as a binary relation between arguments but as a relation between sets of arguments and arguments. Hence, we capture a notion of \textit{premise support} \cite{Cohen2018characterization}.

But there are more situations where accepting some arguments can commit you to accepting another. Suppose there is also an argument $a_5:\arg{s}{r}$. Then, accepting $a_1$ means accepting all premises of $a_5$ and thus $a_5$ should be accepted as well. This is a kind of \textit{premise-sharing-support}, sometimes called \textit{exhaustion} \cite{Ulbricht2024nonflat}. In that sense, our notion of support combines premise-support and premise-sharing-support in order to capture all situations where accepting a set of arguments commits you to accept another one as well.
This leads to the following definition.

\begin{definition}[Support and Attack]\label{def:SuppAtt}
	Let $a,b$ be arguments in language $\lang$. 
	
	We say a set of arguments $E$ supports $a$ if $Prem(a)\subseteq Sent(E)$.
	
	We say that $a$ attacks $b$ if $Conc(a)\in \inc{s}$ for some $s\in Sent(b)$ or if $Conc(a)\in\inc{n(b)}$.
\end{definition}

\begin{example}
	Using the setting of Examples \ref{ex:stradivarius}, \ref{ex:lang} and \ref{ex:args},
	we have a support from $a_1:\arg{Ale}{Str}$ to $a_2:\arg{Str}{Exp}$ and a mutual rebut between $a_4:\arg{Cla}{Ann}$ and $a_5:\arg{Die}{Hil}$. 
\end{example}

The following definition collects all of this together.

\begin{definition}[Structured Bipolar Argumentation Framework]\label{def:sbaf2}
	A structured bipolar argumentation framework (SBAF) is a tuple $\sbaff$ where $A$ is a finite set of arguments in language $\lang$ and $\att,\supp$ are the corresponding attack and support relations.
\end{definition}

We write $a\supp b$ in case $(a,b)\in {\supp}$ and use the same notational conventions as with attacks. The intuitive running example is revisited in Example \ref{ex:full-stradivarius}. Here, we introduce a more technical example.

\begin{example}\label{ex:sbaf}
	The following example is based on $t\in\inc{r}$, $r\in\inc{t}$, $r\in\inc{n(a_6)}$, and $z\in\inc{p}$ (and accordingly $p\in\inc{z}$). Note that $a_5$ undercuts $a_6$ as it attacks its inference claim.
	
	{\centering
		\scalebox{0.95}{
			\begin{tikzpicture}[node distance=0.8cm and 0.6cm]
				
				\node (A1) {$a_1:\l\{s\},s\r$};
				\node (A2) [right=of A1] {$a_2:\l\{u\},v\r$};
				\node (A3) [right=of A2] {$a_3:\l\{w\},x\r$};
				\node (A4) [above=of A1] {$a_4:\l\{s\},t\r$};
				\node (A5) [above=of $(A2)!0.5!(A3)$, yshift=0.25cm] {$a_5:\l\{v,x\},r\r$};
				\node (A6) [right=of A5] {$a_6:\l\{y\},z\r$};
				\node (A7) [right=of A3] {$a_7:\l\{p\},q\r$};
				
				\coordinate (coordinate) at ([yshift=-0.6cm]A5.south);
				
				\draw [s-arrow] (A1) -- (A4);
				\draw [s-arrow] (A4) -- (A1);
				\draw [a-arrow] (A4) -- (A5);
				\draw [a-arrow] (A5) -- (A4);
				\draw [a-arrow] (A5) -- (A6);
				\draw [s-arrow] (A2) -- (coordinate) -- (A5);
				\draw [s-arrow] (A3) -- (coordinate) -- (A5);
				\draw [a-arrow] (A6) -- (A7);
				
			\end{tikzpicture}
		}
	}
\end{example}

It is useful to define (strongly) saturated frameworks. These are such that they contain minimal arguments for sentences that are incompatible with others. This ensures that all important relations between sentences are visible as relations between arguments.

\begin{definition}[Saturated SBAFs]
	An SBAF $\sbaf$ is called saturated (resp.\ strongly saturated) if $\forall s\in Sent(A)\ s.t.\ \exists t\in {Sent(A)}\cap{\overline{s}}$, there is a minimal argument for $s$ or (resp.\ and) for $t$ in $A$, and $\forall u\in Sent(A)\ s.t.\ u\in\overline{n(a)}$ for some $a\in A$, there is a minimal argument for $u$ in $A$.
\end{definition}

The SBAF in Example \ref{ex:sbaf} is not saturated, but could be made so by adding, for instance, $\arg{r}{r}$ and $\arg{z}{z}$. To make it strongly saturated, $\arg{t}{t}$ and $\arg{p}{p}$ would also have to be added.


\subsection{Semantics}\label{sect:sem}

We now develop semantics for SBAFs with two goals in mind: to allow rejecting certain defended arguments, and to evaluate both acceptable sets of arguments and acceptable sets of sentences. This dual perspective assists agents with different capabilities to communicate and resolve potential disagreements. That is, it allows agents that stores their knowledge in form of sentences to communicate and interact with agents that only see arguments.

For argument semantics, we find a middle ground between admissible and complete semantics. One is not forced to accept all defended arguments, but their rejection is limited by support. A simple way of doing this is to start with admissible semantics and to add a condition that makes sure that extensions are closed under support (cf.\ \cite{CayrolLagasquie-Schiex2013bipolarity,Ulbricht2024nonflat}). You should accept an argument not if it is defended, but if you already accept all its premises. 

However, there are cases where this condition may be relaxed. A first exception occurs in case you already accept undercutting information for the supported argument. In such a case, you accept that the inference in the argument does not hold and as such you can reject it, even if you accept all its premises. Another exception can be argued for in case you accept all premises of an undefended argument. Depending on how strongly one interprets the support relation, you might be forced to accept it or not. 

We give two semantics, based on whether the second exception is allowed. This gives a \textit{strong} and a \textit{weak} interpretation of support.

\begin{definition}[Coherent Argument Extensions]\label{def:coherent}
	A strongly coherent argument extension in an SBAF $\sbaff$ is an admissible extension $E\subseteq A$ that satisfies:
	\begin{description}
		\item[Strong Support-Closure:] $\forall a\in A:$ if $E$ supports $a$ and $E$ does not contain undercutting information for $a$, then $a\in E$.
	\end{description}
	
	A weakly coherent argument extension is an admissible extension $E$ that satisfies:
	\begin{description}
		\item[Weak Support-Closure:] $\forall a\in A:$ if $E$ supports $a$, $E$ does not contain undercutting information for $a$, and $E$ defends $a$, then $a\in E$.
	\end{description}
\end{definition}

\begin{observation}
	Strongly coherent extensions are weakly coherent.
\end{observation}

\begin{example}\label{ex:coherent}
	Consider the SBAF of Example \ref{ex:sbaf}. The extension $\{a_1, a_2, a_3, a_4, a_6\}$ is weakly coherent, but not strongly so. According to strong coherence, it is not possible to accept all of $a_1,a_2$, and $a_3$ together since strong support-closure forces acceptance of $a_4$ and $a_5$, violating admissibility.
\end{example}

The difference between strong and weak coherence lies in different answers to the question whether support can force you to accept undefended arguments. Strong coherence says ``yes'', weak coherence says ``no''. Weak coherence allows you to reject an argument if it would violate defence, even if you accept all its premises. In some sense, it allows you to infer that the inference of the argument must be faulty without having explicit undercutting information. Strong coherence, in contrast, requires explicit undercutting information to reject an inference. Without it, the inference is assumed to work.

A more formal way to bring out the difference between strong and weak coherence is the property of directionality \cite{BaronGiacomin2007principle,Yu2023principle}: Arguments should only affect the acceptability of other arguments if they are connected via a directed path through support and attack. This is true for weak coherence, but not for strong coherence.

\begin{description}
	\item[Directionality:] Let $\sbaf$ be an SBAF, $U\subseteq A$ be such that ${A\backslash U}\nsupp {U}$ and $A\backslash U\not\att U$ and define $\sbaf_{|U}=\l\lang, U,{\att}\cap(U\times U),{\supp}\cap(U\times U)\r$. A semantics $\sigma$ satisfies \textit{directionality} if an extension $E\subseteq U$ is acceptable according to $\sigma$ in $\sbaf_{|U}$ iff there exists a $\sigma$-acceptable extension $E'\subseteq A$ in $\sbaf$ such that $E=E'\cap U$.
\end{description}

\begin{proposition}\label{prop:support-properties}
	Strong coherence fails directionality.
	
	Weak coherence satisfies directionality.
\end{proposition}

What does a semantics from the language perspective look like? If we are interested in acceptable sets of sentences, we first require that they do not contain incompatible sentences. We call a set of sentences $S$ \textit{compatible} if $\forall s,t\in S:s\not\in\inc{t}$. But apart from this condition, it is the arguments that limit the choice of sentences we can accept---this is the point of arguing. For this we need to know to which arguments you are committed to when accepting some sentences and vice-versa.

Given an argument extension, we can simply say that the corresponding language extension consists of all sentences of accepted arguments. When we start with a language extension, however, it is less clear how the accepted sentences translate to accepted arguments. For instance, perhaps it is possible to accept all premises and the conclusion of an argument without accepting the argument itself---one can simply disbelief that one can infer the conclusion from the premises. We define two ways of translating from language extensions to argument extensions that are analogous to strong and weak coherence.

The first way, similar to strong coherence, is to count all arguments as accepted of which the premises are accepted and there is no undercutting information. This gives the \textit{strong argument set}. Starting with a set of accepted sentences, it collects together all arguments that should be accepted according to the strong interpretation of support.

\begin{definition}[Strong Argument Set]\label{def:SargSet}
	Given a set of sentences $S$ in $\sbaff$, we define its strong argument set $Arg_s(S):=\{a\in A\ |\ Prem(a)\subseteq S\text{ and }{\uncuts{a}}\cap S=\emptyset\}.$
\end{definition}

Analogously, we can define a \textit{weak argument set}. Given a set of sentences $S$ in an SBAF $\sbaf$, we define its \textit{weak argument set} which collects together all the arguments in $\sbaf$ which should be accepted according to the weak interpretation of support. As we only require accepting defended arguments, we define it using a fixpoint construction analogous to that in abstract argumentation for complete extensions \cite{Dung1995acceptability, Baroni2018abstract}. We first define the analogue of the characteristic function in abstract argumentation.

\begin{definition}[Characteristic Function]\label{def:RespFunct}
	Given an SBAF $\sbaff$ and a set of sentences $S\subseteq Sent(A)$, we define the characteristic function $R^S_{\sbaf}:2^A\to2^A$ as
	$$R^S_{\sbaf}(E):=\{a\in A\ |\ a\in\sarg(S)\text{ and $E$ defends $a$}\}.$$
\end{definition}

With this, we can define the \textit{weak argument set}. 

\begin{definition}[Weak Argument Set]\label{def:WargSet}
	If $S$ is compatible, then its initial set, $Init(S)$, is defined as the largest admissible subset of $\{a\in A\ |\ Sent(a)\subseteq S\text{ and }\uncuts{a}\cap S=\emptyset\}.$ The weak argument set of a compatible $S$, $Arg_w(S)$, is the least fixpoint of $R^S_\sbaf$ containing $Init(S)$.
\end{definition}

We need to make sure that the weak argument set contains $Init(S)$, as otherwise the least fixpoint of $R^S_{\sbaf}$ would sometimes be too small. For instance, in case of two arguments $a_1:\arg{s}{t}$ and $a_2:\arg{u}{u}$ with $u\in\inc{t}$ and $t\in\inc{u}$, the least fixpoint for the language extension $S=\{s,t\}$ would be empty. However, we want it to contain $a_1$ as all its sentences are accepted and it defends itself.

\begin{proposition}\label{prop:WargWD}
	The weak argument set is well-defined.
\end{proposition}

{ One motivation for introducing language extensions is that agents may find it easier to judge which sentences they accept than which arguments. The formal complexity of defining the weak argument set confirms this intuition. Further, it allows us explain why certain arguments are accepted or rejected by reference to their underlying sentences.} 

\begin{example}\label{ex:ArgSet}
	Consider again the SBAF of Example \ref{ex:sbaf} and take the language extension $S=\{s,t,u,v,w,x,y\}$. Its strong argument set is $\sarg(S)=\{a_1,a_2,a_3,a_4,a_5,a_6\}$, which is not conflict-free. For the weak argument set, we have: $Init(S)=\{a_1,a_2,a_3,a_4\}$ and $\warg(S)=\{a_1,a_2,a_3,a_4,a_6\}$.
	
	Consider also $S'=\{s,u,v,w,x,r\}$ with $\warg(S')=\{a_1,a_2,a_3,a_5\}$. While both $S$ and $S'$ end up accepting all of $a_1$, $a_2$, and $a_3$, thus supporting both $a_4$ and $a_5$, they differ in which of them they accept. On the argument level, the situation looks symmetrical, but using the sentence perspective, we can explain this difference through the prior commitment of $S$ to the conclusion of $a_4$ and the prior commitment of $S'$ to that of $a_5$.
\end{example}

\begin{proposition}\label{prop:SargCFWargAdm}
	For a compatible language extension $S$, $\warg(S)$ is admissible.
\end{proposition}

We can now specify which sets of sentences should be accepted in an SBAF. We again get a strong and a weak semantics, depending on whether we take language extensions to commit to their strong or their weak argument set. For both, we require language extensions to include conclusions of accepted arguments. Further, we know from Proposition \ref{prop:SargCFWargAdm} that weak argument sets are defended, but for strong argument sets, we have to require it explicitly.

\begin{definition}[Adequate Language Extensions]\label{def:adeqExt}
	A strongly adequate language extension in an SBAF $\sbaff$ is a compatible set of sentences $S\subseteq Sent(A)$ such that $Arg_s(S)$ defends all its arguments and it satisfies:
	\begin{description}
		\item[Sentence-Closure:] $\forall a\in Arg_s(S):Sent(a)\subseteq S$.
	\end{description}
	A weakly adequate language extension $S\subseteq Sent(A)$ is a compatible set of sentences that satisfies sentence-closure w.r.t.\ $Arg_w(S)$.
\end{definition}

\begin{example}\label{ex:adequate}
	As seen in Example \ref{ex:ArgSet}, $S=\{s,t,u,v,w,x,y\}$ is not strongly adequate, as sentence-closure would require adding $r$ to it, violating compatibility. However, it is weakly adequate, as $\warg(S)$ does not include $a_5$. In contrast, $S'=\{u,v,w,x,r,y,z\}$ has $\sarg(S')=\{a_2,a_3,a_5\}$ and is strongly adequate. Note that $S'$ contains undercutting information for $a_6$ (namely $r$) and does not have to accept it. 
\end{example}

\begin{proposition}\label{prop:strongAdeqAreWeak}
	Strongly adequate language extensions are also weakly adequate.
\end{proposition}

{
	\begin{example}\label{ex:full-stradivarius}
		We can now fully analyse Example \ref{ex:stradivarius}. The language is provided in Example \ref{ex:lang}. We get the following SBAF.
		\[\begin{tikzpicture}[node distance=0.5cm]
			\node (A1) {$a_1:\arg{Ale}{Str}$};
			\node (A2) [right=of A1] {$a_2:\arg{Str}{Exp}$};
			\node (A3) [below=of A1,yshift=0.4cm] {$a_3:\arg{New}{Cla}$};
			\node (A4) [right=of A3] {$a_4:\arg{Cla}{Ann}$};
			\node (A5) [right=of A4] {$a_5:\arg{Die}{Hil}$};
			
			\draw [s-arrow] (A1) -- (A2);
			\draw [s-arrow] (A3) -- (A4);
			\draw [a-arrow] (A4) -- (A5);
			\draw [a-arrow] (A5) -- (A4);
		\end{tikzpicture}\]
		
		We can see that $\{a_3\}$ is weakly, but not strongly coherent, and that $\{Str,Exp,Cla,Hil\}$ is weakly but not strongly adequate. This is because a strongly adequate extension would have to add $Ann$, since $a_4$ is in its strong argument set, but then compatibility would be violated. In this example, strong coherence coincides with d-admissibility (Section \ref{sect:bip} expands on this observation). But weak coherence captures novel dynamics of doubt in argumentation.
		
	\end{example}
}

Adequate language extensions and coherent argument extensions capture similar ideas. But, as the following example illustrates, there are SBAFs where the two notions come apart.

\begin{example}
	Consider an SBAF with two arguments $a_1:\arg{s}{t}$ and $a_2:\arg{u}{v}$ with $s\in\inc{u}$. Then, $\{a_1,a_2\}$ is a strongly (and thus weakly) coherent argument extension, but its set of sentences is not compatible and hence neither strongly nor weakly adequate. 
\end{example}

Nevertheless, we can show that for a large class of SBAFs, there is a direct correspondence between adequate language extensions and coherent argument extensions. { This confirms that even when we are interested in whether an agent accepts an acceptable argument extension, we can evaluate directly their accepted set of sentences.}

\begin{proposition}\label{prop:strong-correspondence}
	Let $\sbaff$ be a saturated SBAF. 
	Then for every strongly adequate language extension $S$, its strong argument set $Arg_s(S)$ is a strongly coherent argument extension.
	
	Also, for every strongly coherent argument extension $E$, its set of sentences $Sent(E)$ is a strongly adequate language extension.
\end{proposition}

\begin{proposition}\label{prop:weak-correspondence}
	Let $\sbaff$ be a saturated SBAF.
	Then for every weakly adequate language extension $S$, its weak argument set $\warg(S)$ is a weakly coherent argument extension.
	
	Also, for every weakly coherent argument extension $E$, its set of sentences $Sent(E)$ is a weakly adequate language extension.
\end{proposition}

This concludes the semantics for SBAF, implementing both the critical reaction of doubt and the language perspective.


\section{Interpreting Abstract and Bipolar Semantics in SBAFs}\label{sect:interpreting}

We now give new interpretations of preferred semantics and deductive support semantics \cite{Boella2010support,CayrolLagasquie-Schiex2013bipolarity}, as they both describe special cases of weakly, resp.\ strongly coherent semantics. Thus, in some circumstances, SBAFs can be seen as instantiating these semantics.


\subsection{SBAFs and Preferred Semantics}\label{sect:abst}

{We now compare our semantics with those of Dung-style abstract argumentation. Those semantics take neither support between arguments nor argument structure into account, so the following results show the circumstances under which we can ignore these aspects of SBAFs.} The following observation shows that Dung-style semantics are related to weak coherence.

\begin{observation}\label{obs:compAreWeakCoher}
	Complete extensions are weakly coherent and a weakly coherent extension is $\subseteq$-maximal iff it is preferred.
\end{observation}

But we can use the language perspective to reach a more interesting result. While preferred extensions correspond to an agent who accepts as many \textit{arguments} as possible, what kind of extensions do we get from an agent who accepts as many \textit{sentences} as possible? We first define \textit{confident} extensions to capture this idea. 

\begin{definition}[Confident Extensions]\label{def:conf-ext}
	Let $\sbaf$ be an SBAF. A language extension $S\subseteq Sent(A)$ is called a confident strongly (resp.\ weakly) adequate language extension if it is $\subseteq$-maximal amongst strongly (resp.\ weakly) adequate language extensions.
	
	An argument extensions $E\subseteq A$ is called a confident strongly (resp.\ weakly) coherent argument extension if it is strongly (resp.\ weakly) coherent and there exists a confident strongly (resp.\ weakly) adequate language extension $S\subseteq Sent(A)$ such that $E=\sarg(S)$ (resp.\ $E=\warg(S)$).
\end{definition}

\begin{example}\label{ex:confident}
	In the SBAF of Example \ref{ex:sbaf}, we have the following confident strongly coherent extensions: $\{a_1, a_2, a_4, a_6\}, \{a_1, a_3, a_4, a_6\},$ and $\{a_2, a_3, a_5, a_7\}$. And the following confident weakly coherent extensions: $\{a_1, a_2, a_3, a_5, a_7\},$ and $\{a_1, a_2, a_3, a_4, a_6\}$. 
\end{example}

We can again note correspondence between argument and language extensions. This time, the direction from arguments to sentences works only indirectly, as the set of sentences of a confident coherent argument extension might not itself be confident adequate.

\begin{proposition}\label{prop:conf-correspondence}
	Let $\sbaf$ be a saturated SBAF.
	For every confident strongly (resp.\ weakly) adequate language extension, its strong (resp.\ weak) argument set is confident strongly (resp.\ weakly) coherent.
	
	Also, for every confident strongly (resp.\ weakly) coherent argument extension $E$, there exists a confident strongly (resp.\ weakly) adequate language extension $S$ such that $\sarg(S)=E$ (resp.\ $\warg(S)=E$).
\end{proposition}

In Example \ref{ex:confident}, confident weakly coherent extensions correspond to preferred extensions. Thus, there, maximising arguments is the same as maximising sentences. But this is not always the case.

\begin{example}\label{ex:weakNotPref}
	In the following SBAF (with $s\in\inc{u}$), $\emptyset$ is confident weakly coherent, based on $\{t,u\}$ being confident weakly adequate, but it is not preferred.
	
	\centering
	\begin{tikzpicture}[node distance=0.5cm]
		
		\node (A1) {$a_1:\l\{s\},s\r$};
		\node (A2) [right=of A1] {$a_2:\l\{t,u\},t\r$};
		
		\draw [a-arrow] (A1) -- (A2);
		
	\end{tikzpicture}
\end{example}

Nevertheless, we find that in saturated SBAFs, preferred extensions are confident weakly coherent. For the other direction, we need strongly saturated frameworks. Thus, it is only in strongly saturated frameworks, where many minimal arguments are added, that maximising arguments and maximising sentences coincide. The correspondence between confident weakly coherent extensions and preferred extensions also shows that, in strongly saturated SBAFs, the former does not take support into account, as the latter can be calculated in pure attack-frameworks.

\begin{proposition}\label{prop:Pref-are-maxWeak}
	Let $\sbaff$ be a saturated SBAF.
	Then any preferred extension $E\subseteq A$ is confident weakly coherent.
\end{proposition}

\begin{proposition}\label{prop:maxWeak-are-pref}
	Let $\sbaff$ be a strongly saturated SBAF.
	Then any confident weakly coherent argument extension $E\subseteq A$ is preferred.
\end{proposition}

\begin{figure}
	\centering
	\includegraphics[scale=0.5]{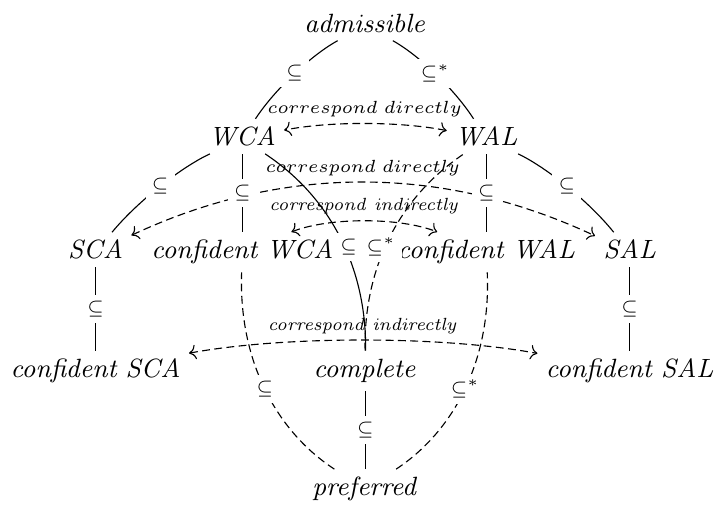}
	
	\caption{$SCA$: strongly coherent argument extensions, $WCA$: weakly coherent argument extensions, $WAL$: weakly adequate language extensions, $SAL$: strongly adequate language extensions. 
		$\mathit{SCA}\subseteq\mathit{WCA}$ indicates that each strongly coherent extension is also weakly coherent. ``$\subseteq^*$'' relates argument extensions with the argument sets of language extensions.
		Dashed relations only hold in saturated SBAFs.}
	\label{fig:Overview}
\end{figure}

{ These two propositions show that at least in a large class of SBAFs, an agent who wants to maximise either arguments or sentences can safely disregard support between arguments and solely focus on attacks. In some sense, then, support becomes redundant if we use preferred-style semantics. While it is an open question whether we can find a full correspondence between weakly coherent and complete semantics, it indicates that we should indeed use admissible-style semantics if we want to account for supports.} Figure \ref{fig:Overview} shows all our semantics and their relations.


\subsection{SBAFs and Deductive Support Semantics}\label{sect:bip}

We base our discussion of bipolar semantics on the notion of \textit{deductive} support \cite{Boella2010support,CayrolLagasquie-Schiex2013bipolarity}. This approach implements the idea that accepting a supporting argument entails accepting the supported argument as well. We show that it is related to strong coherence.

\begin{definition}[Bipolar Argumentation Framework]\label{def:baf}
	A bipolar argumentation framework (BAF) is a tuple $\baff$ where $A$ is a finite set of arguments, ${\att}\subseteq A\times A$ an attack relation, and ${\supp}\subseteq A\times A$ a support relation.
\end{definition}

\begin{definition}[Supported and Mediated Attacks]\label{def:supportedMediatedAttacks}
	Let $\baff$ be a bipolar argumentation framework and take $a,b\in A$. 
	
	We say that there is a supported attack from $a$ to $b$ if there exists $c\in A$ such that $a\supp c$ and $c\att b$.
	
	We say that there is a mediated attack from $a$ to $b$ if there exists $c\in A$ such that $b\supp c$ and $a\att c$.
\end{definition}

\begin{example}\label{ex:CompAtt}
	In the following, BAF, the attack from $a_1$ to $a_3$ is a mediated attack and that from $a_3$ to $a_4$ is a supported attack. Added attacks are indicated with dotted arrows.
	
	\centering
	\begin{tikzpicture}[node distance=0.5cm]
		
		\node (A1) {$a_1$};
		\node (A2) [right=of A1] {$a_2$};
		\node (A3) [below=of A2] {$a_3$};
		\node (A4) [left=of A3] {$a_4$};
		\node (A5) [left=of A1] {$a_5$};
		
		\draw [a-arrow] (A1) -- (A2);
		\draw [s-arrow] (A3) -- (A2);
		\draw [a-arrow] (A2) -- (A4);
		\draw [a-arrow] (A5) -- (A1);
		
		\draw [c-arrow] (A1) -- (A3);
		\draw [c-arrow] (A3) -- (A4);
		
	\end{tikzpicture}
\end{example}

\begin{definition}[Deductive Support Semantics]\label{def:deductiveSupport}
	Let $\baff$ be a bipolar argumentation framework. We define the set of complex attacks, 
	$${\att^{co}}:=\bigcup_{i\in\mathbb{N}}{\att^i},$$
	where $\att^0:=\att$ and $\att^{i+1}:={\att^i}\cup\{(a,b)\in A\times A\ |\ \text{there is a supported or mediated attack from $a$ to $b$ w.r.t.\ ${\att^i}$}\}$.
	
	An extension $E\subseteq A$ is called d-admissible, d-complete, or d-preferred if it is admissible and closed under $\supp$, complete, or preferred in $\af=\l A,\att^{co}\r$.
\end{definition}

\begin{example}
	In the BAF of Example \ref{ex:CompAtt}, we have the unique d-complete, and d-preferred extension $\{a_2, a_3, a_5\}$. On the d-admissible side, we have $\emptyset,\{a_5\},\{a_2,a_5\}$, and $\{a_2,a_3,a_5\}$. Note that $\{a_3,a_5\}$ is not d-admissible since it is not closed under $\supp$.
\end{example}

We can now observe that strong coherence behaves like d-admissibility under certain conditions. Thus, SBAFs can specify exactly the circumstances under which the structure of the arguments does not need to be considered.

\begin{proposition}\label{prop:ded-in-sbafs}
	Let $\sbaff$ be an SBAF where $\forall a\in A:|Prem(a)|=1$ and $\neg\exists t\in Sent(A),\exists a\in A:t\in\uncuts{a}$. 
	Then an extension $E\subseteq A$ that is strongly coherent is also d-admissible.
	
	If further $\forall a,b\in A$, we have $Prem(a)\neq Prem(b)$, then an extension $E\subseteq A$ that is d-admissible is also strongly coherent.
\end{proposition}

This result tells us that strong coherence coincides with d-admissibility if each argument has exactly one premise that is unique and there is no undercutting information in the framework. Accordingly, we only need to take information about the structure of the arguments into account if they are more complex than unique single-premise. Strong coherence can then be seen as a generalisation of d-admissibility to these more complex cases. 

\section{Related Approaches}\label{sect:related}

\subsection{Bipolar Argumentation}

The semantics presented in this paper rely on the notions of strong, resp.\ weak support-closure. Both variants find their analogues in the literature, be it under different background assumptions. 

A version of strong support-closure can be found in the notion of \textit{exhaustion} in \textit{premise-augmented BAFs} (pBAFs) \cite{Ulbricht2024nonflat}. pBAFs go some way towards a structured approach to bipolar argumentation by adding information about the premises of the arguments. This allows keeping track of which premises an argument extension is committed to. A \textit{p-admissible} extension is then required to be admissible and exhaustive, meaning that it contains all arguments of which all premises are accepted. As there are no exceptions made for undefended arguments this corresponds, disregarding undercutting information, to the notion of strong support-closure. 

Crucially, pBAFs do not take the conclusions of arguments into account, only their premises. This can lead to frameworks that work as pBAFs but are uninstantiable if we think about the conclusions. 

\begin{example}\label{ex:pBAF}
	Consider the following pBAF, where $\pi$ indicates the premises of the arguments.\\
	{
		
		\centering
		\begin{tikzpicture}[node distance=0.5cm]
			
			\node (a1) {$a_1$, $\pi(a_1)=\{s\}$};
			\node (a2) [right=of a1] {$a_2$, $\pi(a_2)=\{t\}$};
			\node (a3) [below=of a1] {$a_3$, $\pi(a_3)=\{\neg t\}$};
			\node (a4) [right=of a3] {$a_4$, $\pi(a_4)=\{u\}$};
			
			\draw [s-arrow] (a1) -- (a2);
			\draw [s-arrow] (a1) -- (a3);
			
		\end{tikzpicture}
		
	}
	
	This framework is uninstantiable once we think about the conclusion of $a_1$. Namely, it has to support both $t$ and $\neg t$. This might be possible if the conclusion of $a_1$ is itself contradictory, but then it would have to support everything, including $a_4$. Thus, this framework cannot be instantiated, but we only see it if we consider the conclusions of arguments. This is exactly the information SBAFs add.
\end{example}

The notion of weak support-closure occurs in \textit{bipolar set-argumentation frameworks} (BSAFs) \cite{Berthold2024capturing} in the definition of $\Gamma$-admissibility. While BSAFs do not take the structure of arguments into account, they allow for collective supports, as do SBAFs. Further, it can be shown that in saturated SBAFs, weak coherence does not need to take into account undercutting information explicitly, as this is covered by the defence-requirement. This makes weak coherence analogous to $\Gamma$-admissibility, which requires an extension to contain all supported arguments that are also defended. 

The difference between BSAFs (as well as pBAFs) and SBAFs lies in the notion of defence. Namely, BSAFs and pBAFs require defence only against \textit{closed} sets of arguments, i.e.\ against sets of arguments that are closed under the support relation. SBAFs, in contrast, require defence against all attacking arguments.

\begin{example}
	Consider the following SBAF with $s\in\inc{r}$, $r\in\inc{s}$, and $t\in\uncuts{a_3}$. 
	
	\begin{center}
		\begin{tikzpicture}[node distance=0.5cm]
			
			\node (A1) {$a_1:\l\{s\},t\r$};
			\node (A2) [right=of A1] {$a_2:\l\{r\},r\r$};
			\node (A3) [below=of A2] {$a_3\arg{r}{u}$};
			
			\draw [a-arrow] (A2) -- (A1);
			\draw [a-arrow] (A1) -- (A3);
			
			\draw [s-arrow] (A2) -- (A3);
			
		\end{tikzpicture}
	\end{center}
	
	Argument $a_1$ defends itself against the closed set $\{a_2,a_3\}$ because it attacks $a_3$, but it is not defended against $a_2$. Accordingly, $\{a_1\}$ is not weakly coherent, but it would be $\Gamma$-admissible.
\end{example}

We can further find an implementation of the notion of doubt in \textit{evidential argumentation systems} \cite{OrenNorman2008semantics,PolbergOren2014revisting}. In these frameworks, an argument is only relevant if it is supported by evidence. Any argument that has no evidential support can be rejected. This allows for rejecting unattacked arguments and thus can be seen as an expression of doubt. However, this implementation of doubt is quite different from ours in at least two respects: \textit{First}, doubt in evidential argumentation systems needs to be justified by reference to a lack of evidence and whether an argument lacks evidence is determined by the argumentation system. There is no room for different agents doubting different arguments in the same system, as is possible in SBAFs. In our approach, it is the agents that decide what to doubt, no the framework. We also make no external assumptions on when doubt is justified. \textit{Second} evidential argumentation systems rely on a notion of \textit{necessary} support \cite{NouiouaRisch2011argumentation,Nouioua2013afs}, whereas we showed in Proposition \ref{prop:ded-in-sbafs} that our notion of support generalises \textit{deductive} support. With necessary support, an argument can only be accepted if all its supporters are also accepted. In that sense, accepting the supporter is necessary for accepting the supported argument. In SBAFs it is never required to accept a supporting argument, support can only force acceptance of supported arguments. Thus, the direction goes the different way and accordingly there is no natural correspondence to necessary support in SBAFs.

\subsection{Structured Argumentation}
Our account bears similarities to structured approaches to argumentation such as ASPIC$^+$ \cite{Modgil2013general,Modgil2014aspic,Prakken2010abstract} or ABA \cite{Bondarenko1993assumption,Toni2014tutorial}. Both approaches structure arguments into premises and conclusions, where this structure then determines the relations between the arguments. However, our conception of arguments differs from theirs. \textit{First}, we take the arguments as given, while both ASPIC$^+$ and ABA use knowledge bases (resp.\ assumptions) and a set of inference rules to construct them. This guarantees some rationality in the inferences of the arguments as they must be licenced by some rule. In contrast, SBAF arguments could in principle combine any premises with conclusions, just like abstract argumentation could in principle count any statement as an argument. \textit{Second}, our account of argument structure is less complex than that of ASPIC$^+$ or ABA. Our arguments are simple premise-conclusion structures, whereas arguments in other structured approaches are full inference trees that derive their conclusions by means of potentially many inference rules and intermediary steps from their premises. This difference is brought out by our use of a support relation where ASPIC$^+$ and ABA work with attacks only. An argument that concludes with a premise of another one counts as a supporting argument (or an element of a supporting set) in SBAFs, but would count as a sub-argument of a larger, more complex argument in ASPIC$^+$. \textit{Third}, ABA allows evaluation of frameworks in terms of acceptable sets of sentences, similarly to our language extensions. The difference is that language extensions take account of all sentences that are involved in arguments, where ABA only gives sets of assumptions, which represent the original premises of their inference tree-arguments. 

Finally, claim-augmented argumentation frameworks (CAFs) \cite{Bernreiter2024effect,Dvorak2023claim-centric} associate each argument with a claim, representing their conclusion, and provide semantics that work purely on the level of claims. They also compare the strategies of maximising accepted claims and maximising accepted arguments, and they conclude that these strategies are identical in well-formed CAFs. As SBAFs are well-formed by definition, their result differs from ours: Confident weakly coherent extensions and preferred extensions do not coincide in general, as SBAFs also take into account the premises. 


\section{Conclusion}

Argumentation in practice has many features and aspects, which in turn lead to many ways of modelling it. In this paper, we followed ideas from informal approaches to argumentation and developed an approach to structured bipolar argumentation. This approach is characterised by semantics that are weaker than typical completeness-based semantics in that not all defended arguments have to be accepted and by the correspondence between the argument and the language perspective. Not only are we be able to say which arguments an agent should accept, but we can also directly evaluate which sentences they should accept. As our results show, depending on the properties of the framework (i.e.\ whether it is saturated), the two perspectives can align or come apart. Further, we can compare the strategy of maximising accepted arguments (i.e.\ choosing a preferred extension) with that of maximising accepted sentences (i.e.\ choosing a confident weakly coherent extension). Again, these strategies can lead to different results, but they coincide in strongly saturated frameworks. Amongst the semantics for bipolar argumentation, we find a correspondence to d-admissibility. 

Our approach allows agents to doubt any sentence which they cannot infer from already accepted sentences. This implements a very generous notion of doubt that may make it difficult to force some agent to accept an argument, as they could in principle always doubt our premises. Accordingly, it might be desirable to distinguish between sentences that can be freely doubted and other sentences that belong to a form of common ground that should be accepted by all agents. A strong version of this could be captured in our approach by requiring all extensions to contain the sentences representing the common ground. However, it would be useful to implement a weaker common ground where commonly accepted sentences can be rejected if they are attacked, but have to be accepted otherwise.

The comparison between SBAFs and other approaches to structured argumentation should also be explored further. For instance, it would be interesting to examine how a notion of doubt could be added to frameworks such as ASPIC$^+$ or ABA. A further point of comparison are approaches to argumentation where it is uncertain which arguments are present. There, doubt could be modelled by removing an argument, thus possibly presenting an alternative to our account. Such approaches can be found in probabilistic argumentation \cite{Li2012probabilistic,Hunter2020epistemic} or argumentation with incomplete frameworks \cite{Baumeister2018complexity,Baumeister2021acceptance}.
Finally, it would also be interesting to follow empirical approaches on argumentation \cite{CramervnderTorre2019scf2,CramerGuillame2018directionality,Vesic2022graphical,vanEemeren2009fallacies} and examine whether the informal ideas of allowing for doubt and evaluating frameworks on sentences correspond to the intuitions of ordinary reasoners.	

\section*{Acknowledgements}
	Srdjan Vesic benefited from the support of the project AGGREEY ANR-22-CE23-0005 of the French National Research Agency (ANR). Bruno Yun benefited from the University Claude Bernard, Lyon 1 AAP Accueil funding.

	\bibliographystyle{plain} 
	\bibliography{all_references}

\begin{thebibliography}{10}

\bibitem{AmgoudCayrol1998on}
Leila Amgoud and Claudette Cayrol.
\newblock {On the Acceptability of Arguments in Preference-based
  Argumentation}.
\newblock In Gregory~F. Cooper and Serafín Moral, editors, {\em Uncertainty in
  Artificial Intelligence (UAI98)}, page 1–7. Morgan Kaufmann Publishers,
  1998.

\bibitem{AmgoudVesic2009repairing}
Leila Amgoud and Srdjan Vesic.
\newblock {Repairing Preference-Based Argumentation Frameworks}.
\newblock In Craig Boutilier, editor, {\em International Joint Conference on
  Artificial Intelligence (IJCAI09)}, page 665–670. IJCAI Organization, 2009.

\bibitem{Arieli2021logic}
Ofer Arieli, AnneMarie Borg, Jesse Heyninck, and Christian Straßer.
\newblock {Logic-Based Approaches to Formal Argumentation}.
\newblock {\em Journal of Applied Logics}, 8(6):1793–1898, 2021.

\bibitem{Baroni2018abstract}
Pietro Baroni, Martin Caminada, and Massimiliano Giacomin.
\newblock {Abstract Argumentation Frameworks and Their Semantics}.
\newblock In Pietro Baroni, Dov Gabbay, Massimiliano Giacomin, and Leendert
  {Van Der Torre}, editors, {\em Handbook of Formal Argumentation. Volume 1}.
  College Publications, 2018.

\bibitem{Baroni2018handbook}
Pietro Baroni, Dov Gabbay, Massimiliano Giacomin, and Leendert van~der Torre,
  editors.
\newblock {\em {Handbook of Formal Argumentation. Volume 1}}.
\newblock College Publications, 2018.

\bibitem{BaronGiacomin2007principle}
Pietro Baroni and Massimiliano Giacomin.
\newblock {On principle-based evaluation of extension-based argumentation
  semantics}.
\newblock {\em Artificial Intelligence}, 171(10-15):675–700, 2007.

\bibitem{Baumeister2021acceptance}
Dorothea Baumeister, Matti Järvisalo, Daniel Neugebauer, Andreas Niskanen, and
  Jörg Rothe.
\newblock {Acceptance in incomplete argumentation frameworks}.
\newblock {\em Artificial Intelligence}, 295:103470, 2021.

\bibitem{Baumeister2018complexity}
Dorothea Baumeister, Daniel Neugebauer, Jörg Rothe, and Hilmar Schadrack.
\newblock {Complexity of verification in incomplete argumentation frameworks}.
\newblock {\em Proceedings of the AAAI Conference on Artificial Intelligence
  (AAAI18)}, 32(1):1753–1760, 2018.

\bibitem{Bernreiter2024effect}
Michael Bernreiter, Wolfgang Dvořák, Anna Rapberger, and Stefan Woltran.
\newblock The effect of preferences in abstract argumentation under a
  claim-centric view.
\newblock {\em Journal of Artificial Intelligence Research}, 81:203–262,
  2024.

\bibitem{Berthold2024capturing}
Matti Berthold, Anna Rapberger, and Markus Ulbricht.
\newblock Capturing non-flat assumption-based argumentation with bipolar
  setafs.
\newblock In Pierre Marquis, Magdalena Ortiz, and Maurice Pagnucco, editors,
  {\em Principles of Knowledge Representation and Reasoning (KR24)}, page
  128–133. Association for the Advancement of Artificial Intelligence, 2024.

\bibitem{Besnard2014introduction}
Philippe Besnard, Alejandro~J. García, Anthony Hunter, Sanjay Modgil, Henry
  Prakken, Guillermo~R. Simari, and Francesca Toni.
\newblock {Introduction to structured argumentation}.
\newblock {\em Argument and Computation}, 5(1):1–4, 2014.

\bibitem{Betz2010theorie}
Gregor Betz.
\newblock {\em {Theorie dialektischer Strukturen}}.
\newblock Klostermann, 2010.

\bibitem{Boella2010support}
Guido Boella, Dov Gabbay, Leendert van~der Torre, and Serena Villata.
\newblock {Support in Abstract Argumentation}.
\newblock In Pietro Baroni, Federico Cerutti, Massimiliano Giacomin, and
  Guillermo~R. Simari, editors, {\em Computational Models of Argument
  (COMMA10)}, page 111–122. IOS Press, 2010.

\bibitem{Bondarenko1993assumption}
Andrei Bondarenko, Francesca Toni, and Robert~Anthony Kowalski.
\newblock {An assumption-based framework for non-monotonic reasoning}.
\newblock In Luís~Moniz Pereira and Anil Nerode, editors, {\em Proceedings of
  the second international workshop on Logic programming and non-monotonic
  reasoning}. MIT Press, 1993.

\bibitem{CayrolLagasquie-Schiex2005on}
Claudette Cayrol and Marie-Christine Lagasquie-Schiex.
\newblock {On the acceptability of arguments in bipolar argumentation
  frameworks}.
\newblock In Lluis Godo, editor, {\em Symbolic and Quantitative Approaches to
  Reasoning with Uncertainty (ECSQARU05)}, page 378–389. Springer, 2005.

\bibitem{CayrolLagasquie-Schiex2013bipolarity}
Claudette Cayrol and Marie-Christine Lagasquie-Schiex.
\newblock {Bipolarity in argumentation graphs: Towards a better understanding}.
\newblock {\em International Journal of Approximate Reasoning}, 54:876–899,
  2013.

\bibitem{CocarascuToni2017mining}
Oana Cocarascu and Francesca Toni.
\newblock Mining bipolar argumentation frameworks fron natural language text.
\newblock In Floris Bex, Floriana Grasso, and Nancy Green, editors, {\em
  Workshop on Computational Models of Natural Argument (CMNA@ICAIL17)}, pages
  65--70. CEUR, 2017.

\bibitem{Cohen2018characterization}
Andrea Cohen, Simon Parsons, Elizabeth~I. Sklar, and Peter McBurney.
\newblock {A characterization of types of support between structured arguments
  and their relationship with support in abstract argumentation}.
\newblock {\em International Journal of Approximate Reasoning}, 94:76–104,
  2018.

\bibitem{Corsi2025attack}
Esther~Anna Corsi.
\newblock Attack principles in sequent-based argumentaiton.
\newblock {\em Journal of Logic and Computation}, 35(3), 2025.

\bibitem{CramerGuillame2018directionality}
Marcos Cramer and Mathieu Guillaume.
\newblock {Directionality of attacks in natural language argumentation}.
\newblock {\em CEUR Workshop Proceedings}, 2261:40–46, 2018.

\bibitem{CramervnderTorre2019scf2}
Marcos Cramer and Leendert van~der Torre.
\newblock {SCF2 – An argumentation semantics for rational human judgments on
  argument acceptability}.
\newblock In Christoph Beierle, Marco Ragni, Frieder Stolzenburg, and Matthias
  Thimm, editors, {\em Proceedings of the 8th Workshop on Dynamics of Knowledge
  and Belief (DKB-2019) and the 7th Workshop KI \& Kognition (KIK-2019)}, page
  24–35, 2019.

\bibitem{Dung1995acceptability}
Phan~Minh Dung.
\newblock {On the acceptability of arguments and its fundamental role in
  nonmonotonic reasoning, logic programming and n-person games}.
\newblock {\em Artificial Intelligence}, 77:321–357, 1995.

\bibitem{Dvorak2023claim-centric}
Wolfgang Dvořák, Anna Rapberger, and Stefan Woltran.
\newblock A claim-centric perspective on abstract argumentation semantics:
  Claim-defeat, principles, and expressiveness.
\newblock {\em Artificial Intelligence}, 324:104011, \#no 2023.

\bibitem{Gonzalez2021labeled}
Melisa {Esca Ñuela Gonzalez}, Maximiliano Budán, Gerardo Simari, and
  Guillermo Simari.
\newblock Labeled bipolar argumentation frameworks.
\newblock {\em Journal of Artificial Intelligence Research}, 70:1557–1636,
  2021.

\bibitem{Hamblin1970fallacies}
Charles~Leonard Hamblin.
\newblock {\em {Fallacies}}.
\newblock Methuen and Co, 1970.

\bibitem{Hunter2020epistemic}
Anthony Hunter, Sylwia Polberg, and Matthias Thimm.
\newblock Epistemic graphs for representing and reasoning with positive and
  negative influences of arguments.
\newblock {\em Artificial Intelligence}, 281(103236), 2020.

\bibitem{Johnson2006logical}
Ralph~H. Johnson and J.~Anthony Blair.
\newblock {\em {Logical Self-Defense}}.
\newblock International Debate Education Association, 2006.

\bibitem{Koszowy2022pragmatic}
Marcin Koszowy, Katarzyna Budzynska, Barbara Konat, Steve Oswald, and Pascal
  Gygax.
\newblock {A Pragmatic Account of Rephrase in Argumentation: Linguistic and
  Cognitive Evidence}.
\newblock {\em Informal Logic}, 42(1):49–82, 2022.

\bibitem{KrabbeVanLaar2011ways}
Erik~C.W. Krabbe and Jan~Albert van Laar.
\newblock {The ways of criticism}.
\newblock {\em Argumentation}, 25(2):199–227, 2011.

\bibitem{Li2012probabilistic}
Hengfei Li, Nir Oren, and Timothy~J. Norman.
\newblock {Probabilistic argumentation frameworks}.
\newblock In Sanjay Modgil, Nir Oren, and Francesca Toni, editors, {\em Theorie
  and Applications of Formal Argumentation (TAFA 2011)}, page 1–16. Springer,
  2012.

\bibitem{Modgil2013general}
Sanjay Modgil and Henry Prakken.
\newblock {A general account of argumentation with preferences}.
\newblock {\em Artificial Intelligence}, 195:361–397, 2013.

\bibitem{Modgil2014aspic}
Sanjay Modgil and Henry Prakken.
\newblock {The ASPIC+ framework for structured argumentation: A tutorial}.
\newblock {\em Argument and Computation}, 5(1):31–62, 2014.

\bibitem{Nouioua2013afs}
Farid Nouioua.
\newblock {AFs with Necessities: Further Semantics and Labelling
  Characterization}.
\newblock In Weiru Liu, V.~S. Subrahmanian, and Jef Wijsen, editors, {\em
  Scalable Uncertainty Management (SUM13)}, page 120–133. Springer, 2013.

\bibitem{NouiouaRisch2011argumentation}
Farid Nouioua and Vincent Risch.
\newblock {Argumentation frameworks with necessities}.
\newblock In Salem Benferhat and John Grant, editors, {\em Scalable Uncertainty
  Management (SUM11)}, page 163–176. Springer, 2011.

\bibitem{OrenNorman2008semantics}
Nir Oren and Timothy~J. Norman.
\newblock {Semantics for Evidence-Based Argumentation}.
\newblock In Philippe Besnard, Sylvie Doutre, and Anthony Hunter, editors, {\em
  Computational Models of Argument (COMMA08)}, page 276 – 284. IOS Press,
  2008.

\bibitem{PolbergHunter2018empirical}
Sylwia Polberg and Anthony Hunter.
\newblock {Empirical evaluation of abstract argumentation: Supporting the need
  for bipolar and probabilistic approaches}.
\newblock {\em International Journal of Approximate Reasoning}, 93:487–543,
  2018.

\bibitem{PolbergOren2014revisting}
Sylwia Polberg and Nir Oren.
\newblock {Revisiting Support in Abstract Argumentation Systems}.
\newblock In Simon Parsons, Nir Oren, Chris Reed, and Federico Cerutti,
  editors, {\em Computational Models of Argument (COMMA 2014)}, pages 369--376.
  IOS Press, 2014.

\bibitem{Potyka2018continous}
Nico Potyka.
\newblock {Continuous Dynamical Systems for Weighted Bipolar Argumentation}.
\newblock In {\em Principles of Knowledge Representation and Reasoning (KR18)},
  page 148–157. Association for the Advancement of Artificial Intelligence,
  2018.

\bibitem{Prakken2010abstract}
Henry Prakken.
\newblock {An abstract framework for argumentation with structured arguments}.
\newblock {\em Argument and Computation}, 1(2):93–124, 2010.

\bibitem{Pu2015attacker}
Fuan Pu, Jian Luo, Yulai Zhang, and Guiming Luo.
\newblock {Attacker and Defender Counting Approach for Abstract Argumentation}.
\newblock In David~C. Noelle, Rick Dale, Anne~S. Warlaumont, Jeff Yoshimi,
  Teenie Matlock, Carolyn~D. Jennings, and Paul~P. Maglio, editors, {\em
  Proceedings of the 37th Annual Meeting of the Cognitive Science Society
  (CogSci15)}, page 1913–1918, 2015.

\bibitem{Toni2014tutorial}
Francesca Toni.
\newblock {A tutorial on assumption-based argumentation}.
\newblock {\em Argument and Computation}, 5(1):89–117, 2014.

\bibitem{Toulmin2003uses}
Stephen~E. Toulmin.
\newblock {\em {The Uses of Argument: Updated Edition}}.
\newblock Cambridge University Press, 2003.

\bibitem{Ulbricht2024nonflat}
Markus Ulbricht, Nico Potyka, Anna Rapberger, and Francesca Toni.
\newblock {Non-flat ABA Is an Instance of Bipolar Argumentation}.
\newblock In {\em Proceedings of the AAAI Conference on Artificial Intelligence
  (AAAI24)}, volume~38, page 10723–10731, 2024.

\bibitem{vanEemeren2009fallacies}
Frans van Eemeren, Bart Garssen, and Bert Meuffels.
\newblock {\em Fallacies and Judgements of Reasonableness. Empirical Research
  Concerning the Pragma-Dialectical Discussion Rules}.
\newblock Springer, 2009.

\bibitem{VanEemeren1984speech}
Frans~H. van Eemeren and Rob Grootendorst.
\newblock {\em {Speech acts in argumentative discussions. A theoretical model
  for analysis of discussions directed towards solving conflicts of opinion}}.
\newblock Foris Publications, 1984.

\bibitem{VanEemeren2004systematic}
Frans~H. van Eemeren and Rob Grootendorst.
\newblock {\em {A Systematic Theory of Argumentation. The pragma-dialectical
  approach}}.
\newblock Cambridge University Press, 2004.

\bibitem{VanLaarKrabbe2012burden}
Jan~Albert van Laar and Erik~C.W. Krabbe.
\newblock {The Burden of Criticism: Consequences of Taking a Critical Stance}.
\newblock {\em Argumentation}, 27(2):201–224, 2013.

\bibitem{Vesic2022graphical}
Srdjan Vesic, Bruno Yun, and Predrag Teovanovic.
\newblock {Graphical Representation Enhances Human Compliance with Principles
  for Graded Argumentation Semantics}.
\newblock In Catherine Pelachaud, Matthew~E. Taylor, Piotr Faliszewski, and
  Viviana Mascardi, editors, {\em Autonomous Agents and Multiagent Systems
  (AAMAS22)}, page 1319 – 1327. International Foundation for Autonomous
  Agents and Multiagent Systems, 2022.

\bibitem{Walton2008argumentation}
Douglas Walton, Chris Reed, and Fabrizio Macagno.
\newblock {\em {Argumentation Schemes}}.
\newblock Cambridge University Press, 2008.

\bibitem{Yu2023principle}
Liuwen Yu, Caren {Al Anaissy}, Srdjan Vesic, Xu~Li, and Leendert van~der Torre.
\newblock {A Principle-Based Analysis of Bipolar Argumentation Semantics}.
\newblock In Sarah Gaggl, Maria~Vanina Martinez, and Magdalena Ortiz, editors,
  {\em Logics in Artificial Intelligence (JELIA23)}, page 209–224. Springer,
  2023.

\end{thebibliography}
	
	
	\newpage
	\appendix
	
	\section{Proofs}
	\setcounter{theorem}{1}
	
	\begin{proposition}
		Strongly coherent extensions fail directionality.
		
		Weakly coherent extensions satisfy directionality.
	\end{proposition}
	\begin{proof}
		Figure \ref{fig:noDirectForStrong} gives the counterexample for directionality of strongly coherent extensions. For weakly coherent extensions, first let $E\subseteq U$ be such that $E=E'\cap U$ for some weakly coherent $E'\subseteq A$. Then $E$ clearly satisfies weak support-closure. Further, it is also admissible, since $E$ is conflict-free and defended against any attacks within $U$. Thus $E$ is weakly coherent. Now we assume that $E$ is weakly coherent in $\sbaf_{|U}$. Note that the closure of $E$ under $R^{Sent(E)}_\sbaf$ is weakly coherent, as shown by Lemma \ref{lem:R-pres-adm}. Thus there exists a weakly coherent extension $E'\subseteq A$ such that $E=E'\cap U$.
		
	\end{proof}
	
	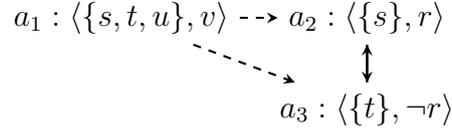
\begin{figure}[h]
		\centering
		\begin{tikzpicture}[node distance=0.5cm]
			
			\node (A1) {$a_1:\arg{s,t,u}{v}$};
			\node (A2) [right=of A1] {$a_2:\arg{s}{r}$};
			\node (A3) [below=of A2] {$a_3:\arg{t}{\neg r}$};
			
			\draw [s-arrow] (A1) -- (A2);
			\draw [s-arrow] (A1) -- (A3);
			\draw [a-arrow] (A2) -- (A3);
			\draw [a-arrow] (A3) -- (A2);
			
		\end{tikzpicture}
		
		\caption{$\{a_1\}$ is strongly coherent if we disregard arguments $a_2$ and $a_3$. However, in the presence of those arguments, accepting $a_1$ would, by strong support-closure, require to accept both $a_2$ and $a_3$. But then the extension would not be conflict-free. Thus, there is no strongly coherent extension containing $a_1$ in the full framework.}
		\label{fig:noDirectForStrong}
	\end{figure}
	
	\begin{lemma}\label{lem:R-monotonic}
		Let $\sbaff$ be an SBAF, $S\subseteq Sent(A)$, and $E,E'\subseteq A$.
		If $E\subseteq E'$, then $R^S(E)\subseteq R^S(E').$
	\end{lemma}
	\begin{proof}
		Take any $E, E'\subseteq A$ such that $E\subseteq E'$. Let $a\in R^S(E)$. Then we know that $a\in\sarg(S)$, and $E$ defends $a$. We immediately get $E'$ defends $a$, which is sufficient for $a\in R^S(E')$ as desired. 
	\end{proof}
	
	\begin{lemma}\label{lem:init}
		Let $\sbaff$ be an SBAF and $S\subseteq Sent(A)$ a compatible language extension. 
		Then there exists a unique maximal admissible subset of $\{a\in A\ |\ Sent(a)\subseteq S\text{ and }\inc{n(a)}\cap S=\emptyset\}$.
	\end{lemma}
	\begin{proof}
		We first show that $\{a\in A\ |\ Sent(a)\subseteq S\text{ and }\uncuts{a}\cap S=\emptyset\}$ is conflict-free. Suppose it contains argument $a,b$ such that $a\to b$. Then either $Conc(a)\in\inc{s}$ for some $s\in Sent(b)$, which contradicts compatibility of $S$, or $Conc(a)\in\inc{n(b)}$, contradicting that $\inc{n(b)}\cap S=\emptyset$. Thus, the set is conflict-free.
		
		Now we show that the union of all admissible subsets of the set is admissible, thus clearly being its unique maximal admissible subset. Since the whole set is conflict-free, so are all subsets and thus also their union. Further, the union of defended sets is defended, thus the union of all admissible subsets is also defended and hence admissible itself.
	\end{proof}
	
	\begin{lemma}\label{lem:ArgwEx}
		Let $\sbaff$ be an SBAF and $S\subseteq Sent(A)$ a compatible language extension.
		Then there exists a least fixpoint of $R^S$ containing $Init(S)$.
	\end{lemma}
	\begin{proof}
		The set $\{E\subseteq A\ |\ Init(S)\subseteq E\}$ is a complete lattice. Further, Lemma \ref{lem:R-monotonic} gives that $R^S$ is a monotonic function. Thus the claim follows from the Knaster-Tarski fixpoint theorem.
	\end{proof}
	
	\begin{proposition}
		The weak argument set is well-defined.
	\end{proposition}
	\begin{proof}
		By Lemmas \ref{lem:init} and \ref{lem:ArgwEx}.
	\end{proof}
	
	\begin{lemma}\label{lem:R-pres-adm}
		Let $\sbaff$ be an SBAF, $S\subseteq Sent(A)$.
		Then for any admissible extension $E\subseteq A$ such that for all $a\in E$: $a\in\sarg(S)$, we have that $E\subseteq R^{S}(E)$ and  $R^S(E)$ is admissible.
	\end{lemma}
	\begin{proof}
		Let $E$ be admissible and $a\in\sarg(S)$ for all $a\in E$. Take any $a\in E$, then, since $E$ is admissible, it defends $a$, thus $a\in R^S(E)$ as desired.
		
		Note that $E\subseteq R^{S}(E)$ gives us that $R^S(E)$ is defended. Further, $R^S(E)$ is a subset of the set of defended arguments by $E$, which is known to be conflict-free. Thus, $R^S(E)$ is also conflict-free.
	\end{proof}
	
	\begin{lemma}\label{lem:Argw-admissible}
		Let $\sbaff$ be an SBAF and $S\subseteq Sent(A)$ a compatible language extension.
		Then for each $i\in\mathbb{N}$, $R^S_i(Init(S))\subseteq R^S_{i+1}(Init(S))$, and $R^S_{i+1}(Init(S))$ is admissible.
	\end{lemma}
	\begin{proof}
		We proceed by induction on $i$. 
		
		For the base case, we know that $Init(S)$ is admissible and it is clear by its construction that for all $a\in Init(S)$, we have $Prem(a)\subseteq S$ and $\inc{n(a)}\cap S=\emptyset$, i.e.\ $a\in\sarg(S)$. Thus Lemma \ref{lem:R-pres-adm} gives us $Init(S)\subseteq R^S(Init(S))$ and $R^S(Init(S))$ is admissible. 
		
		For the induction step, note that Lemma \ref{lem:R-pres-adm} applies to all $R^S_{i+1}(Init(S))$, as $R^S_{i}(Init(S))$ is admissible by the induction hypothesis and by definition of $R^S$, we have that for all $a\in R^S_{i}(Init(S))$, $a\in\sarg(S)$.
	\end{proof}
	
	\begin{lemma}\label{prop:ArgwConst}
		Let $\sbaff$ be an SBAF and $S\subseteq Sent(A)$ a compatible language extension.
		Then $$Arg_w(S)=\bigcup_{i\in\mathbb{N}}R^S_i(Init(S)).$$
	\end{lemma}
	\begin{proof}
		Note that Lemma \ref{lem:Argw-admissible} gives us that $\bigcup_{i\in\mathbb{N}}R^S(R^S_i(Init(S)))=\bigcup_{i\in\mathbb{N}}R^S_i(Init(S))$. Thus we can prove that $\bigcup_{i\in\mathbb{N}}R^S_i(Init(S))$ is a fixpoint by showing that $$R^S(\bigcup_{i\in\mathbb{N}}R^S_i(Init(S)))=\bigcup_{i\in\mathbb{N}}R^S(R^S_i(Init(S))).$$ 
		\begin{description}
			\item[$\subseteq:$] Let $a\in R^S(\bigcup_{i\in\mathbb{N}}R^S_i(Init(S)))$. Then $a\in\sarg(S)$, and $a$ is defended by $\bigcup_{i\in\mathbb{N}}R^S_i(Init(S))$. Since, by Lemma \ref{lem:Argw-admissible}, $R^S_i(Init(S))\subseteq R^S_{i+1}(Init(S))$, this gives us some $i$ such that $a$ is defended by $R^S_i(Init(S))$, meaning that $a\in R^S(R^S_i(Init(S)))$ and also $a\in \bigcup_{i\in\mathbb{N}}R^S(R^S_i(Init(S)))$ as desired. 
			
			\item[$\supseteq:$] Let $a\in \bigcup_{i\in\mathbb{N}}R^S(R^S_i(Init(S)))$. Again, by Lemma \ref{lem:Argw-admissible}, this gives us some $i$ such that $a\in R^S(R^S_i(Init(S)))$, meaning that $a\in\sarg(S)$, and $a$ is defended by $R^S_i(Init(S))$. The latter gives us that $a$ is defended by $\bigcup_{i\in\mathbb{N}}R^S_i(Init(S))$ and in sum we have $a\in R^S(\bigcup_{i\in\mathbb{N}}R^S_i(Init(S))).$
		\end{description}
		It remains to show that $\bigcup_{i\in\mathbb{N}}R^S_i(Init(S))$ is the least fixpoint containing $Init(S)$. Assume for a contradiction that there is some fixpoint $E$ such that $E\subsetneq \bigcup_{i\in\mathbb{N}}R^S_i(Init(S))$. By Lemma \ref{lem:Argw-admissible}, this gives us some $i$ such that $E\subsetneq R^S_i(Init(S))$. Note that we also have $Init(S)\subsetneq E$, since otherwise $Init(S)$ would itself be the least fixpoint and we would have $Init(S)=\bigcup_{i\in\mathbb{N}}R^S_i(Init(S))$. But then, by $i$ applications of Lemma \ref{lem:R-monotonic}, $R^S_i(Init(S))\subseteq R^S_i(E)=E\subsetneq R^S_i(Init(S))$, a contradiction.
		
		Thus, we conclude that $\bigcup_{i\in\mathbb{N}}R^S_i(Init(S))$ is the least fixpoint containing $Init(S)$.
		
	\end{proof}
	
	\begin{proposition}
		For a compatible language extension $S$, $\warg(S)$ is admissible.
	\end{proposition}
	\begin{proof}
		By Lemmas \ref{lem:R-pres-adm} and \ref{prop:ArgwConst}.
	\end{proof}
	
	\begin{proposition}
		Strongly adequate language extensions are also weakly adequate.
	\end{proposition}
	\begin{proof}
		Let $S$ be a strongly adequate language extension. We show that $\sarg(S)=\warg(S)$. It is immediate that $Init(S)\subseteq\sarg(S)$. Further, since $R^S$ is monotonic (Lemma \ref{lem:R-monotonic}) and $R^S(\sarg(S))\subseteq \sarg(S)$ (strong support-closure) we also have $\warg(S)\subseteq\sarg(S)$. Now consider any $a\in\sarg(S)$. Then by sentence-closure, we have $Sent(a)\subseteq S$ and by the conditions of $\sarg(S)$ also $\uncuts{a}\cap S=\emptyset$. Thus $a\in \{a\in A\ |\ Sent(a)\subseteq S\text{ and }\uncuts{a}\cap S=\emptyset\}$. This gives $\sarg(S)\subseteq\{a\in A\ |\ Sent(a)\subseteq S\text{ and }\uncuts{a}\cap S=\emptyset\}$ and since $\sarg(S)$ is admissible (see part of proof of Proposition \ref{prop:weak-correspondence} which does not depend on saturatedness), we further have $\sarg(S)\subseteq Init(S)\subseteq \warg(S)$.
	\end{proof}
	
	\begin{proposition}
		Let $\sbaff$ be a saturated SBAF. 
		Then for every strongly adequate language extension $S$, its strong argument set $Arg_s(S)$ is a strongly coherent argument extension.
		
		Also, for every strongly coherent argument extension $E$, its set of sentences $Sent(E)$ is a strongly adequate language extension.
	\end{proposition}
	\begin{proof}
		Let $S$ be a strongly adequate language extension. We check all conditions for strong coherence of $\sarg(S)$.
		\begin{description}
			\item[Conflict-Free:] Suppose there are arguments $a,b\in \sarg(S)$ such that $a\att b$. Since $b\in \sarg(S)$, we know that $\inc{n(b)}\cap S=\emptyset$, hence, by sentence-closure, $Conc(a)\not\in\inc{n(b)}$. Thus we know that $Conc(a)\in Sent(b)$, but then sentence-closure leads to a violation of compatibility of $S$. Thus, $\sarg(S)$ is conflict-free.
			\item[Defence:] This is given by definition.
			\item[Strong Support-Closure:] First, we show that $Sent(\sarg(S))\subseteq S$. So first take any $s\in Sent(\sarg(S))$. Then there exists some $a\in \sarg(S)$ such that $s\in Sent(a)$. By sentence-closure, we have $Sent(a)\subseteq S$, meaning that $s\in S$ as desired. Now suppose that for some $a\in A$, we have that $\sarg(S)$ supports $a$ and does not contain undercutting information, i.e.\ $Prem(a)\subseteq Sent(\sarg(S))$ and $\uncuts{a}\cap Sent(\sarg(S))=\emptyset$. Since $Sent(\sarg(S))\subseteq S$, we directly have $Prem(a)\subseteq S$. It remains to show that $\uncuts{a}\cap S=\emptyset$. Suppose for a contradiction that there is some $t\in\uncuts{a}\cap S$. By saturatedness of $\sbaf$, there is a minimal argument $b\in A$ for $t$. Since $t\in S$, we also know that $b\in\sarg(S)$. But then $t\in Sent(\sarg(S))$, contradicting that $\uncuts{a}\cap Sent(\sarg(S))=\emptyset$. Thus, we conclude that $\uncuts{a}\cap S=\emptyset$ and hence that $a\in\sarg(S)$ as desired.
		\end{description}
		
		Now let $E$ be a strongly coherent argument extension. We check all conditions for strong adequacy of $Sent(E)$.
		\begin{description}
			\item[Compatibility:] Suppose there are $s,t\in Sent(E)$ such that $s\in\inc{t}$. Then there are argument $a,b\in E$ such that $s\in Sent(a)$ and $t\in Sent(b)$. Further, by saturatedness of $\sbaf$, there is a minimal argument $c$ for either $s$ or $t$. Since we assume for minimal arguments that $\uncuts{c}=\emptyset$, strong support-closure gives us $c\in E$. But then either $c\att a$ or $c\att b$, contradicting conflict-freeness of $E$. Thus we conclude that $Sent(E)$ is compatible.
			\item[Defence:] We show that $\sarg(Sent(E))=E$, from which defence follows directly. Thus take $a\in\sarg(Sent(E))$. Then we know that $Prem(a)\subseteq Sent(E)$ and $\uncuts{a}\cap Sent(E)=\emptyset$. Strong support-closure then gives $a\in E$ as desired. Now take $a\in E$. Then we know that $Prem(a)\subseteq Sent(E)$. We need to show that $E$ does not contain undercutting information, i.e.\ $\uncuts{a}\cap Sent(E)=\emptyset$ in order to get $a\in\sarg(Sent(E))$. Thus assume there is some $t\in\uncuts{a}\cap Sent(E)$. By saturatedness of $\sbaf$, we get a minimal argument $b$ for $t$. Strong support-closure gives $b\in E$, but since $b\att a$, this contradicts conflict-freeness of $E$. We conclude that $\uncuts{a}\cap Sent(E)=\emptyset$ and that $a\in\sarg(Sent(E))$ as desired. In sum, $\sarg(Sent(E))=E$.
			\item[Sentence-Closure:] Recall that $\sarg(Sent(E))\subseteq E$. Thus for any $a\in\sarg(Sent(E))$, we know that $a\in E$ and also that $Sent(a)\subseteq Sent(E)$, meaning that $Sent(E)$ satisfies sentence-closure.
		\end{description}
	\end{proof}
	
	\begin{lemma}\label{lem:quick-respect}
		Let $\sbaff$ be a saturated SBAF, $E$ an admissible extension, and $a\in A$ any argument.
		If $E$ defends $a$, then $E$ does not contain undercutting information, i.e.\ $\uncuts{a}\cap Sent(E)=\emptyset$.  
	\end{lemma}
	\begin{proof}
		Suppose that $E$ defends $a$. Further, suppose for a contradiction that $E$ contains undercutting information, i.e.\ there is some $t\in\uncuts{a}\cap Sent(E)$. By saturatedness of $\sbaf$, there is a minimal argument $b$ for $t$. Since $E$ defends $a$, we have $E\att b$. This means there is some argument $c\in E$ with $Conc(c)\in\inc{t}$. But we also have that $t\in Sent(E)$, so there is an argument $d\in E$ with $t\in Sent(d)$. But then $c\att d$, contradicting conflict-freeness of $E$.
	\end{proof}
	
	\begin{lemma}\label{lem:E-fixpoint}
		Let $\sbaff$ be a saturated SBAF.
		Then for any weakly coherent argument extension $E$, we have $R^{Sent(E)}(E)=E$.
	\end{lemma}
	\begin{proof}
		Weak support-closure of $E$ gives us directly that $R^{Sent(E)}(E)\subseteq E$. For the other direction, take any $a\in E$. Then $Prem(a)\subseteq Sent(E)$ and $E$ defends $a$, thus by Lemma \ref{lem:quick-respect}, we have that $\uncuts{a}\cap Sent(E)=\emptyset$, meaning that $a\in R^{Sent(E)}(E)$ as desired.
	\end{proof}
	
	\begin{lemma}\label{lem:warg=init}
		Let $\sbaff$ be a SBAF and $S$ a weakly adequate language extension.
		Then $\warg(S)=Init(S)$.
	\end{lemma}
	\begin{proof}
		We know by definition that $Init(S)\subseteq\warg(S)$. For the other direction, we show that $\warg(S)\subseteq\{a\in A\ |\ Sent(a)\subseteq S\text{ and }\uncuts{a}\cap S=\emptyset\}$. Take any $a\in\warg(S)$. Then by sentence-closure, $Sent(a)\subseteq S$. Now suppose there is some $t\in\uncuts{a}\cap S$. Then certainly $a\not\in Init(S)$. Recall that $\warg(S)=\bigcup_{i\in\mathbb{N}}R^S_i(Init(S))$ (Lemma \ref{prop:ArgwConst}). Thus there is some $i$ such that $a\not\in R^S_i(Init(S))$, but $a\in R^S_{i+1}(Init(S))$. However, $R^S_{i+1}=R^S(R^S_i(Init(S)))=\{a\in A\ |\ a\in\sarg(S)\text{, and $R^S_i(Init(S))$ defends $a$}\}$. Since $\inc{a}\cap S\neq\emptyset$ means $a\not\in\sarg(S)$, we have $a\not\in R^S_{i+1}(Init(S))$, a contradiction. Thus we conclude that $\uncuts{a}\cap S=\emptyset$, meaning that $a\in \{a\in A\ |\ Sent(a)\subseteq S\text{ and }\uncuts{a}\cap S=\emptyset\}$. Thus also $\warg(S)\subseteq\{a\in A\ |\ Sent(a)\subseteq S\text{ and }\uncuts{a}\cap S=\emptyset\}$, and since $\warg(S)$ is admissible (Proposition \ref{prop:SargCFWargAdm}), we have $\warg(S)\subseteq Init(S)$ as desired.
	\end{proof}
	
	\begin{proposition}
		Let $\sbaff$ be a saturated SBAF.
		Then for every weakly adequate language extension $S$, its weak argument set $\warg(S)$ is a weakly coherent argument extension.
		
		Also, for every weakly coherent argument extension $E$, its set of sentences $Sent(E)$ is a weakly adequate language extension.
	\end{proposition}
	\begin{proof}
		Let $S$ be a weakly adequate language extension. We check all conditions of $\warg(S)$. 
		\begin{description}
			\item[Conflict-Free:] Follows from Proposition \ref{prop:SargCFWargAdm}.
			\item[Defence:] Follows from Proposition \ref{prop:SargCFWargAdm}.
			\item[Weak Support-Closure:] Assume for some $a\in A$ that $\warg(S)$ supports $a$, i.e.\ $Prem(a)\subseteq Sent(\warg(S))$, does not contain undercutting information, i.e.\ $\uncuts{a}\cap Sent(\warg(S))=\emptyset$, and defends $a$. We need to show that $a\in\warg(S)$. First note that $Sent(\warg(S))\subseteq S$, as Lemma \ref{lem:warg=init} shows that $\warg(S)=Init(S)$ and all sentences occuring in arguments of $Init(S)$ are already contained in $S$. Thus we have $Prem(a)\subseteq S$. Further, we can show that $\uncuts{a}\cap S=\emptyset$. Suppose for a contradiction that there is some $t\in\uncuts{a}\cap S$. By saturatedness of $\sbaf$, there is a minimal argument $b$ for $t$, for which we know that $b\att a$. Since $\warg(S)$ defends $a$, we know that there is some $c\in\warg(S)$ such that $c\att b$. By $b$ being a minimal argument, we know that $Conc(c)\in\inc{t}$, and we further know by sentence-closure that $Conc(c)\in S$. But this contradicts compatibility of $S$, since also $t\in S$. Thus we conclude that $\uncuts{a}\cap S=\emptyset$. Now we can use that $\warg(S)$ is a fixpoint of $R^S$ to conclude that, since $\warg(S)$ also defends $a$, we have $a\in\warg(S)$ as desired.
		\end{description}
		
		Now let $E$ be a weakly coherent argument extension. We check all conditions for $Sent(E)$.
		\begin{description}
			\item[Compatibility:] Suppose there are $s,t\in Sent(E)$ such that $s\in\inc{t}$. Then we have arguments $a,b\in E$ such that $s\in Sent(a)$ and $t\in Sent(b)$. Further, by saturatedness of $\sbaf$, there is w.l.o.g.\ a minimal argument $c$ for $s$ (note that $c\att b$). Now, we need to additionally show that $E$ defends $c$. Note that since $s\in Sent(a)$, any attack on $c$ is also an attack on $a$. And since $E$ defends $a$, $E$ thus also defends $c$. Weak support-closure then gives $c\in E$. But since $c\att b$, this contradicts conflict-freeness of $E$. Thus we conclude that $Sent(E)$ is compatible.
			\item[Defence:] We show that $E=\warg(Sent(E))$, from which defence follows directly. By Lemma \ref{lem:E-fixpoint}, we know that $R^{Sent(E)}=E$, thus it suffices to show that $E=Init(S)$ (since then $Init(S)$ will itself be the smallest fixpoint containing it).
			\begin{description}
				\item[$\subseteq$:] Take any $a\in E$. Then we have $Sent(a)\subseteq Sent(E)$. Further, we know that $E$ defends $a$ and $\sbaf$ is saturated, thus by Lemma \ref{lem:quick-respect}, we know that $\uncuts{a}\cap Sent(E)=\emptyset$. This gives us that $a\in \{a\in A\ |\ Sent(a)\subseteq Sent(E)\text{ and }\uncuts{a}\cap Sent(E)=\emptyset\}$. But note that we have just now shown that $E\subseteq \{a\in A\ |\ Sent(a)\subseteq Sent(E)\text{ and }\uncuts{a}\cap Sent(E)=\emptyset\}$, and since $E$ is admissible, we can directly infer that $E\subseteq Init(Sent(E))$, as the latter is the largest admissible subset of $\{a\in A\ |\ Sent(a)\subseteq Sent(E)\text{ and }\uncuts{a}\cap Sent(E)=\emptyset\}$.
				
				\item[$\supseteq$:] Take any $a\in Init(Sent(E))$. Then we have that $Sent(a)\subseteq Sent(E)$ and $\uncuts{a}\cap Sent(E)=\emptyset$. If we can show that $E$ defends $a$, then weak support-closure gives us the desired $a\in E$. Thus take any attacker $b$ of $a$. There are two cases. (1) $Conc(b)\in\inc{s}$ for some $s\in Sent(a)$. Since $Sent(a)\subseteq Sent(E)$, there is some argument $c\in E$ such that $s\in Sent(c)$. But then $b\att c$, and since $E$ defends $c$, we also have $E\att b$. That is, $E$ defends $a$ against $b$. (2) $Conc(b)\in\uncuts{a}$. By saturatedness of $\sbaf$, there is a minimal argument $c$ for $Conc(b)$. Note that $c\att a$ and since $a$ is defended by $Init(Sent(E))$, there is some argument $d\in Init(Sent(E))$ such that $d\att c$. This gives us in particular $Conc(d)\in Sent(E)$ and since $c$ is a minimal argument, we also have $Conc(d)\in\inc{Conc(b)}$. Further, $Conc(d)\in Sent(E)$ gives us some argument $e\in E$ such that $Conc(d)\in Sent(e)$. Recall that incompatibility is symmetric, thus $b\att e$. Finally, since $E$ defends $e$, we have $E\att b$, that is, $E$ defends $a$ against $b$.
				
				In sum, $E$ defends $a$ and weak support-closure gives us $a\in E$ as desired.
			\end{description}
			\item[Sentence-Closure:] Take any $a\in\warg(Sent(E))$. Since $\warg(Sent(E))\subseteq E$, we have $a\in E$ and thus $Sent(a)\subseteq Sent(E))$ as desired.
		\end{description}
	\end{proof}
	
	\setcounter{theorem}{8}
	\begin{proposition}
		Let $\sbaff$ be an SBAF. A language extension $S\subseteq Sent(A)$ is called a confident strongly (resp.\ weakly) adequate language extension if it is $\subseteq$-maximal amongst strongly (resp.\ weakly) adequate language extensions.
		
		An argument extensions $E\subseteq A$ is called a confident strongly (resp.\ weakly) coherent argument extension if it is strongly (resp.\ weakly) coherent and there exists a confident strongly (resp.\ weakly) adequate language extension $S\subseteq Sent(A)$ such that $E=\sarg(S)$ (resp.\ $E=\warg(S)$).
	\end{proposition}
	\begin{proof}
		The first part of the proposition follows directly from Propositions \ref{prop:strong-correspondence} and \ref{prop:weak-correspondence}.
		
		The second part is immediate by definition.
	\end{proof}
	
	\begin{proposition}
		Let $\sbaff$ be a saturated SBAF.
		Then any preferred extension $E\subseteq A$ is confident weakly coherent.
	\end{proposition}
	\begin{proof}
		Let $E$ be a preferred extension. Since it is then also complete, Observation \ref{obs:compAreWeakCoher} gives us that it is weakly coherent. It thus remains to show that there exists a confident weakly adequate language extension $S$ such that $E=\warg(S)$.
		
		Consider the set $\{S\subseteq Sent(A)\ |\ E=\warg(S)\text{ and $S$ is weakly adequate}\}$. Note that it is non-empty, since $Sent(E)$ is weakly adequate and $\warg(Sent(E))=E$ (Proposition \ref{prop:weak-correspondence}). Recall that we only consider finite frameworks, thus there exists a $\subseteq$-maximal element $S'$ that is weakly adequate and $\warg(S')=E$. It remains to show that $S'$ is also $\subseteq$-maximal amongst weakly adequate language extensions.
		
		Take any weakly adequate $S''$ such that $S'\subsetneq S''$. Since $S'$ is $\subseteq$-maximal amongst weakly adequate extensions with $\warg(S)=E$, we know that $\warg(S'')\neq E$. We first show that $E\subsetneq\{a\in A\ |\ Sent(a)\subseteq S''\text{ and }\uncuts{a}\cap S''=\emptyset\}$. Take any $a\in E$. Then clearly, $Sent(a)\subseteq Sent(E)\subseteq S'\subsetneq S''$. Now suppose there is some $t\in\uncuts{a}\cap S''$. Then, by saturatedness, there exists a minimal argument $b$ for $t$. Note that $b\att a$ and since $E$ is defended, there is some $c\in E$ such that $c\att b$, that is, $Conc(s)\in\inc{t}$. But since $Sent(c)\subseteq Sent(E)\subseteq S'\subseteq S''$, this contradicts compatibility of $S''$. Hence, $\uncuts{a}\cap S''=\emptyset$ and we can note that $E\subsetneq\{a\in A\ |\ Sent(a)\subseteq S''\text{ and }\uncuts{a}\cap S''=\emptyset\}$. Since $E$ is admissible, this gives $E\subseteq Init(S'')\subseteq \warg(S'')$.
		
		In sum, $E\subsetneq \warg(S'')$, but since $\warg(S'')$ is admissible (Proposition \ref{prop:weak-correspondence}), this contradicts that $E$ is preferred. Hence, there exists not weakly adequate $S''$ such that $S'\subsetneq S''$ and we conclude that $E$ is confident weakly coherent.
	\end{proof}
	
	\begin{proposition}
		Let $\sbaff$ be a strongly saturated SBAF.
		Then any confident weakly coherent argument extension $E\subseteq A$ is preferred.
	\end{proposition}
	\begin{proof}
		Let $E$ be weakly coherent and suppose it is not preferred. We show that $E$ is not confident. We need to show that for any $S\subseteq Sent(A)$ such that $\warg(S)=E$, there exists an adequate language extension $S'$ such that $S\subsetneq S'$. 
		
		Since $E$ is admissible, there exists a preferred extension $E'$ such that $E\subsetneq E'$. By Proposition $\ref{prop:Pref-are-maxWeak}$, we know that $E'$ is confident weakly coherent. Hence, there exists a confident weakly adequate language extension $S'$ such that $\warg(S')=E'$. 
		
		Now take any $S\subseteq Sent(A)$ such that $\warg(S)=E$. We show that $S\subsetneq S'$. Take any $s\in Sent(A)\backslash S'$. Then $s\not\in S'$ and since $S'$ is confident, we know that $S'\cup\{s\}$ is not weakly adequate. We can show that $S'\cup\{s\}$ is not compatible. Suppose it is. Then $\warg(S'\cup\{s\})$ is admissible (Proposition \ref{prop:SargCFWargAdm}), and we know that $E'\not\subseteq\warg(S'\cup\{s\})$ (since otherwise either $S'\cup\{s\}$ would be weakly adequate, if $E'=\warg{S'\cup\{s\}}$, or $E'$ not preferred, if $E\subsetneq \warg{S'\cup\{s\}}$). This lets us take an argument $a\in E'\backslash (S'\cup\{s\})$ and we know that $s\in\uncuts{a}$ (since otherwise $a\in Init(S'\cup\{s\})$). By saturatedness of $\sbaf$, we then have a minimal argument $b$ for $s$. Note that $b\att a$ and since $E'$ is defended, there is some $c\in E'$ such that $Conc(c)\in\inc{s}$. But note that $Sent(c)\subseteq Sent(E)= Sent(\warg(S'))\subseteq S'$ (see proof of Proposition \ref{prop:weak-correspondence}), contradicting compatibility of $S'\cup\{s\}$. Thus, we know that $S'\cup\{s\}$ cannot be compatible. But then there is some $t\in S'$ such that $t\in\inc{s}$. By strong saturatedness, there exists a minimal argument $d$ for $t$. Since it is a minimal argument, $\{d\}$ is admissible and hence $d\in Init(S')\subseteq E'$. Since $d\att b$, we also have $E'\att b$. But since $E\subseteq E'$, we know that $b\not\in E$ (since otherwise $E'$ would not be conflict-free). Finally, since $b$ is a minimal argument for $s$, we know that $s\not\in S$ (since otherwise $b\in\warg(S)=E$). In sum, $S\subseteq S'$.
		
		It remains to note that $S\neq S'$, since otherwise $E=\warg(S)=warg(S')=E'$. Thus $S\subsetneq S'$ and $E$ is not confident. This establishes the result, since we know that any preferred $E\subseteq A$ is weakly coherent.
	\end{proof}
	
	\begin{proposition}
		Let $\sbaff$ be an SBAF where $\forall a\in A:|Prem(a)|=1$ and $\neg\exists t\in Sent(A),\exists a\in A:t\in\uncuts{a}$. 
		Then an extension $E\subseteq A$ that is strongly coherent is also d-admissible.
		
		If further $\forall a,b\in A$, we have $Prem(a)\neq Prem(b)$, then an extension $E\subseteq A$ that is d-admissible is also strongly coherent.
	\end{proposition}
	\begin{proof}
		Let $E$ be strongly coherent. We first check closure under $\supp$. Suppose there is some $a\in E$ and $b\in A$ such that $a\supp b$. Since $|Prem(a)|=1$, we know that $Prem(a)\subseteq Sent(E)$ and since there are no undercuts, we know that $\uncuts{a}\cap Sent(E)=\emptyset$. Thus, by strong support-closure, we have $b\in E$ and $E$ is closed under $\supp$.
		
		Now we show inductively that $E$ is admissible according to $\att^{ded}$. We know by definition that $E$ is admissible w.r.t.\ $\att$, hence it remains to show that $E$ is admissible w.r.t.\ $\att^{i+1}$, assuming it is admissible w.r.t.\ $\att^{i}$. We show defence first. Take any $a\in E$ such that there exists some $b\in A$ with $b\att^{i+1} a$. There are three cases: (i) if also $b\att^i a$, then we have $E\att^i b$ by assumption. (ii) if $b$ supported attacks $a$, then there exists $c\in A$ such that $b\supp c$ and $c\att^i a$. By assumption, we have $E\att^i c$ and thus $E$ mediated attacks $b$, i.e.\ $E\att^{i+1} b$. (iii) if $b$ mediated attacks $a$, then there exists $c\in A$ such that $a\supp c$ and $b\att^i c$. By closure under $\supp$, we have $c\in E$ and by assumption $E\att^i b$, thus also $E\att^{i+1} b$. In sum, $E$ is defended. For conflict-freeness, take any $a,b\in E$. We know that $a\not\att^i b$ by assumption. Further, since both arguments are in $E$ and $E$ is closed under $\supp$, both a supported or a mediated attack from $a$ to $b$ would contradict conflict-freeness of $E$ under $\att^i$. In sum, $E$ is admissible for each $\att^i$ and thus also for $\att^{ded}$. Together with closure under $\supp$, this means that $E$ is d-admissible.
		
		Now assume that $\forall a,b\in A$, we have $Prem(a)\neq Prem(b)$ and let $E$ be d-admissible. It is clear that $E$ is admissible in $\sbaf$. It remains to check strong support-closure. Thus take any $a\in A$ such that $Prem(a)\in Sent(E)$ (recall that there are no undercuts). Our assumption guarantees that then there is some $b\in E$ such that $b\supp a$, and by closure under $\supp$, we get $a\in E$ as desired.
	\end{proof}
	The second direction indeed requires the extra condition. If we just have two arguments, $a_1:\l\{s\},t\r$ and $a_2:\l\{s\},u\r$, it would be d-admissible to just accept $a_1$. However, it would not be strongly coherent, because it would commit to accepting $s$ and thus all premises (and no undercuts) of $a_2$. Hence, strongly coherent extensions would have to be either empty or contain both arguments.
	
\end{document}